\documentclass[11pt,letterpaper]{article}

\usepackage{hyperref}
\hypersetup{
     colorlinks   = true,
     linkcolor = black,
     urlcolor    = teal,
	 citecolor = teal
}
\usepackage[margin=1in]{geometry}
\usepackage{amssymb,amsthm,amsmath,amssymb,wrapfig,dsfont,authblk,mathrsfs}
\usepackage{times}
\usepackage[dvipsnames]{xcolor}
\usepackage{enumitem}
\usepackage{amsfonts}
\usepackage[langfont=typewriter,funcfont=slant]{complexity}
\usepackage{framed}
\usepackage[ruled]{algorithm2e}
\usepackage{soul}
\usepackage{thmtools,thm-restate}
\usepackage{csquotes}
\usepackage[nameinlink, noabbrev, capitalize]{cleveref}
\usepackage{booktabs}       \usepackage{nicefrac}       \usepackage{microtype}
\usepackage{placeins}

\usepackage{nicematrix}
\usepackage{multirow,multicol}
\usepackage{hhline}

\newtheorem{theorem}{Theorem}[section]
\newtheorem{lemma}[theorem]{Lemma}

\newtheorem{corollary}[theorem]{Corollary}
\newtheorem{proposition}{Proposition}[section]
\newtheorem{definition}{Definition}[section]
\newtheorem{remark}{Remark}[section]

\newcommand{\cX}{\mathcal{X}}
\newcommand{\cN}{\mathcal{N}}
\newcommand{\cY}{\mathcal{Y}}
\newcommand{\cH}{\mathcal{H}}
\newcommand{\cG}{\mathcal{G}}
\newcommand{\cD}{\mathcal{D}}
\newcommand{\cQ}{\mathcal{Q}}
\newcommand{\cE}{\mathcal{E}}

\newcommand{\hy}{\widehat{y}}

\newcommand{\regret}{\textsc{Regret}}

\DeclareMathOperator*{\Ex}{\mathbb{E}}

\newcommand{\sD}{\mathscr{D}}
\newcommand{\sQ}{\mathscr{Q}}
\newcommand{\OPT}{\mathsf{OPT}}
\newcommand{\opt}{\mathsf{opt}}
\newcommand{\unif}{\mathcal{U}}
\newcommand{\indicator}[1]{\mathbf{1}\!\left\{#1\right\}}
\newcommand\myatop[2]{\genfrac{}{}{0pt}{}{#1\hfill}{#2\hfill}}
\newcommand{\rel}{\textbf{Rel}}
\newcommand{\simiid}{\overset{\text{iid}}{\sim}}

\newcommand{\stepa}[1]{\overset{\rm (a)}{#1}}
\newcommand{\stepb}[1]{\overset{\rm (b)}{#1}}
\newcommand{\stepc}[1]{\overset{\rm (c)}{#1}}
\newcommand{\stepd}[1]{\overset{\rm (d)}{#1}}

\usepackage{commath} 

\usepackage{autonum}
\usepackage{cleveref}

\crefname{lemma}{lemma}{lemmas}
\Crefname{Lemma}{Lemma}{Lemmas}

\crefname{ineq}{inequality}{inequalities}
\Crefname{Ineq}{Inequality}{Inequalities}
\creflabelformat{ineq}{#2{\upshape(#1)}#3} 
\creflabelformat{Ineq}{#2{\upshape(#1)}#3} 

\crefname{definition}{definition}{definitions}
\Crefname{definition}{Definition}{Definitions}
\crefname{prop}{proposition}{propositions}
\Crefname{Prop}{Proposition}{Propositions}

\begin{document}
\title{Oracle-Efficient Online Learning for Beyond Worst-Case  Adversaries\thanks{An extended abstract of this work was published under the title ``Oracle-efficient Online Learning for Smoothed Adversaries'' in the Proceedings of the $36$th Conference on Neural Information Processing Systems~\cite{haghtalab2022oracle}}}
\author{Nika Haghtalab} \author{Yanjun Han} \author{Abhishek Shetty} \author{Kunhe Yang} 

\affil{University of California, Berkeley\\{\small\texttt{\{nika,yjhan,shetty,kunheyang\}@berkeley.edu}}}
\date{}
\maketitle

\allowdisplaybreaks

\begin{abstract}
  In this paper, we study oracle-efficient algorithms for beyond worst-case analysis of online learning.
  We focus on two settings. First, the smoothed analysis setting of~\cite{NIPS2011_4262,haghtalab2022smoothed} where an adversary
  is constrained to generating samples from distributions whose density is upper bounded by $1/\sigma$ times the uniform density. Second, the setting of $K$-hint transductive learning, where the learner is given access to $K$ hints per time step that are guaranteed to include the true instance. 
  We give the first known oracle-efficient algorithms for both settings that depend only on the pseudo (or VC) dimension of the class and parameters $\sigma$ and $K$ that capture the power of the adversary. 
  In particular, we achieve oracle-efficient regret bounds of  $ \widetilde{O} (  \sqrt{T d\sigma^{-1}} ) $ and $ \widetilde{O} (  \sqrt{T dK} ) $ for learning real-valued functions and $ O (  \sqrt{T d\sigma^{-\frac{1}{2}} }  )$ for learning binary-valued functions.
For the smoothed analysis setting, our results give the first oracle-efficient algorithm for online learning with smoothed adversaries~\cite{haghtalab2022smoothed}. This contrasts the computational separation between online learning with worst-case adversaries and offline learning established by 
\cite{HazanKoren}.
{Our algorithms also achieve improved bounds for worst-case setting with small domains.} In particular, we give an oracle-efficient algorithm with regret of $O ( \sqrt{T(d \abs{\cX})^{1/2} })$, which is a refinement of the earlier $O ( \sqrt{T\abs{\cX}})$ bound by \cite{daskalakis2016learning}.  

   \end{abstract}

\section{Introduction}

Adversarial online learning is a cornerstone of modern machine learning and has led to significant advances in computer science broadly.
A recent line of work on ``beyond the worst-case analysis'' of online learning has brought into light the overly pessimistic nature of standard characterizations of online learnability~\cite{NIPS2011_4262, Gupta_Roughgarden, haghtalab2020smoothed,haghtalab2022smoothed}.
This is exemplified by the results of \cite{haghtalab2022smoothed} showing that adversarial online learnability is \emph{statistically} as easy as PAC learnability, in presence of noise. That is, under \emph{smoothed analysis}, online and offline learnability are both characterized by the finiteness of the VC dimension of a hypothesis class as opposed to the much larger Littlestone dimension that characterizes online learnability in the worst-case~\cite{ben2009agnostic}.
However, to fully deliver on the promise revealed by these statistical insights, there needs to be an algorithmic framework for realizing this connection between online and offline learnability. In this paper, we ask 
\begin{quote}
\centering
\emph{
whether efficient offline learning algorithms lead to efficient online learning algorithms with comparable regret guarantees, when the adversary is not worst-case, such as under the smoothed analysis framework?
}
\end{quote}

\paragraph{Online Learning beyond the Worst-Case.}
Smoothed analysis is a perspective on algorithm design, introduced by \cite{ST04} and formalized for online learning by~\cite{rakhlin2012relax,haghtalab2020smoothed}, in which the adversary is restricted to generating an instance at every round from a distribution that is not overly concentrated, i.e., a distribution whose density is upper bounded by $1/\sigma$ times that of the uniform distribution\footnote{
    While we use the uniform distribution as the base measure for ease of exposition,    
    our results also generalize to arbitrary known base measures.}.
The smoothness of the adversary's actions captures the noise and imprecision inherent in the real world. {A related perspective is that of \emph{transductive} learning that constrains the adversary to producing its sequence from a known set of $T$ unlabeled instances.
We work with a generalization of the transductive learning framework, which we call \emph{transductive learning with $K$ hints}, that constrains the adversary to forming a $T$-long sequence from a known  unordered set of size $KT$. As $K$ and $1/\sigma$ decrease, so does the uncertainty of the learner about the future instances or their distributions.
These  models gracefully capture the expressivity of worst-case instances while circumventing the overly pessimistic nature of the worst-case analysis.}

\paragraph{Computation and Efficiency.}
The question of whether offline learning algorithms  can lead to online learning algorithms is naturally captured by the \emph{oracle-efficiency} framework~(e.g., \cite{Oracle_Efficient,FTPL,HazanKoren}). In this setting, we have access to an offline learning algorithm or equivalently an \emph{empirical risk minimization (ERM)} oracle  which can compute an optimal hypothesis given any history of the actions of the adversary{, using $O(1)$ computation. Efficient algorithms must then be designed to tap into the existing ERM oracle, using polynomial number of calls and computation.} 

Our main goal is to design oracle-efficient online algorithms whose regret { bounds resemble the statistically optimal regret bounds as much as possible. In particular, under smoothed analysis and $K$-hint transductive learning, 
these bounds must be characterized by} offline statistical complexity measures, such as the VC dimension or pseudo dimension of a hypothesis class.
Interestingly, \cite{HazanKoren} showed that such {computationally algorithms cannot exist for fully worst-case adversaries. 
Therefore, designing oracle-efficient online algorithms for adversaries who are not fully worst-case, must \emph{simultaneously overcome both statistical and computational impossibilities.}
{This is what we achieve by the main contributions of this paper.}}

\subsection{Main Results}
\allowdisplaybreaks

\begin{table}
    \centering
    \footnotesize
    \begin{tabular}{|c|c|c|c|c|c|}\hline
    &\multicolumn{2}{|c|}{Problem Class} & Regret Bound &Method&Reference\\\hline
    \multirow{3}{*}{\parbox[t]{2cm}{Statistical\\Upper Bound}}&
    \multirow{2}{*}{\parbox[t]{1.5cm}{Smoothed \\ Online\\ Learning}}&
    {Binary}&{$\widetilde{O}\left(\sqrt{dT\ln(\sigma^{-1}) }\right)$} &\multirow{2}{*}{\parbox[t]{2.5cm}{$\epsilon$-Net and \\ prob. coupling}}&\cite[Thm 3.1]{haghtalab2022smoothed}\\
    \hhline{~~--~-}
    &&{Real-values} &{$\widetilde{O}\left(\sqrt{dT\ln(\sigma^{-1}) }\right)$}&& Thm~\ref{thm:real-statistical}\\
    \hhline{~-----}
    &{\parbox[t]{1.5cm}{$K$-hint\\ Transductive\\ Learning}}&
    Real/Binary&{$\widetilde{O}\left(\sqrt{dT\ln (K) }\right)$} &{\parbox[t]{2.5cm}{ 
    $\epsilon$-Net and \\property of hints}}&Thm~\ref{thm:transductive-statistical}\\
    \hline
    \multirow{3}{*}{\parbox[t]{2cm}{Conputational\\Upper Bound}}&
    \multirow{2}{*}{\parbox[t]{1.5cm}{Smoothed \\ Online\\ Learning}}&
    {Binary}&{$\widetilde{O}\left(\sqrt{ d  T \sigma^{-1/2} }  \right)$} &{\parbox[t]{2.7cm}{Poissonized FTPL\\ (1 oracle call /round)}}&Thm~\ref{thm:FTPL}\\
    \hhline{~~----}
    &&{Real/Binary}&{$\widetilde{O}\left(\sqrt{dT\sigma^{-1} }\right)$}&\multirow{2}{*}{\parbox[t]{2.7cm}{Relax-and-Randomize\\(2 oracle calls/round)}}& Thm~\ref{thm:regret-real}\\
    \hhline{~---~-}
    &{\parbox[t]{1.5cm}{$K$-hint\\ Transductive\\ Learning}}&
    {Real/Binary}&{$\widetilde{O}\left(\sqrt{dTK }\right)$} &&Thm~\ref{thm:regret-transductive}\\
    \hline
    \multirow{3}{*}{\parbox[t]{2cm}{Conputational\\Lower Bound}}&
    \multirow{2}{*}{\parbox[t]{1.5cm}{Smoothed \\ Online\\ Learning}}&
    {Alg-independent}&{$\Omega\left(\sqrt{T(d/\sigma)^{1/2}}\right)$} &{\parbox[t]{2.7cm}{Construction for\\ runtime $o(\sqrt{d/\sigma})$}}&Thm~\ref{thm:comp_lower_bound}\\
    \hhline{~~----}
    &&Algorithm (\ref{alg:real-valued},\ref{alg:FTPL})&{$\Omega\left(\sqrt{dT\sigma^{-1/2}}\right)$}&\multirow{2}{*}{}& \multirow{2}{*}{Thm~\ref{thm:lower_bound}}\\
    \hhline{~---~~}
    &{\parbox[t]{1.5cm}{$K$-hint\\ Transductive\\ Learning}}&
    Algorithm (\ref{alg:transductive_k_hints})&{$\Omega\left(\sqrt{dTK^{1/2}}\right)$} &&\\
    \hhline{======}
    \multirow{2}{*}{\parbox[t]{2cm}{Classical \\ settings}}&\multicolumn{2}{|c|}{Small-domain}&$O \left( \sqrt{T(d\abs{\cX})^{1/2}}\right)$&\multirow{2}{*}{Poissonized FTPL}&\multirow{2}{*}{Cor~\ref{cor:smalldomain}}\\
    \hhline{~---~~}
    &\multicolumn{2}{|c|}{Transductive learning}&$O \left( T^{3/4}d^{1/4}\right)$&&\\\hline
    \end{tabular}
    \caption{\small{This table summarized the main results of this paper for smoothed and transductive adversaries. Here, $d$ represents the pseudo-dimension or VC dimension of the hypothesis class $\cH$, $\sigma$ is the smoothness parameter, $K$ is the parameter for the number of hints in the transductive learning, and $T$ is the number of time steps.}}
      \label{table-of-results}
\end{table}

Our work considers two settings in the beyond worst case analysis of online learning: Smoothed analysis of online learning and $K$-hint transductive learning.
In both cases, we give the first oracle-efficient online learning algorithms whose regret is characterized by the statistical offline complexity measures. In particular, we show that there are efficient algorithms, given access to an ERM oracle, that achieve sublinear regret that depends only on the  pseudo- (or VC) dimension of a class of hypothesis, as well as a measure of parameters of these models that capture the power of the adversary, i.e., $\sigma$ and $K$ respectively. 
{We study both the real-valued and binary valued losses and achieve nearly tight upper and lower bound for these settings.}
We summarize our main results in \cref{table-of-results}.

{\paragraph{Upper bounds and Algorithms.}
For the general case, we design an algorithm based on the  \emph{Relax-and-Randomize} principle of \cite{rakhlin2012relax} that achieves a} regret bound of 
$O(\sqrt{TdK})$ for $K$-hint transductive learning and
$ O ( \sqrt{T d \sigma^{-1} }) $ for smoothed online learning, where $d$ is the pseudo-dimension of the hypothesis class. This algorithm uses $2$ oracle calls per round.
We improve these regret bounds for the binary classification setting under smoothed analysis and achieve regret of $ O( \sqrt{T d \sigma^{-1/2} }  )$. The algorithm that achieves these improved bounds is a variant of FTPL that uses Poisson random variables in the design of its perturbations. This algorithm uses $1$ oracle call per round.
{The primary reason we can achieve improved regret in the binary setting is that the $\sigma$-smoothness over the instance domain $\cX$ also implies $\sigma/2$-smoothness over the instance-label pairs $\cX\times\cY$.}

{While these bounds demonstrate sublinear regret that only depends on $d$, $K$ and $\sigma$, their dependence on $K$ and $\sigma$ does not match the (non-efficient) statistically optimal regret bound of
$\widetilde{O}(\sqrt{Td\ln(K)})$ and
$\widetilde{O}(\sqrt{Td\ln(1/\sigma})$. For the case of binary classification under smoothed analysis, the statistically optimal regret bound is due to \cite{haghtalab2022smoothed}. For the other settings, we provide the optimal statistical rates in \cref{appendix:real-valued}.
}

\paragraph{Lower bounds.}{
We further investigate the gap between the computational and statistical regret upper bounds.
We present an algorithm-independent lower bound for the setting of smoothed online learning that shows that any algorithm with runtime $o(\sqrt{d/\sigma})$ will incur a  $\Omega(\sqrt{T(d/\sigma)^{1/2}})$ regret. We note that this lower bound demonstrates the same dependence on $1/\sigma$ as our upperbound (for the binary classification setting under smoothed analysis).
We also provide algorithm-dependent regret lower bounds that demonstrate improved dependence on parameters $K$ and $d$, and apply to all relax-and-randomize and FTPL-style algorithms.}

{\paragraph{Improved Bounds for the Classical Settings.}
In addition to constrained adversaries, our algorithms and analysis also imply improved regret bounds for two classical settings, namely, \emph{online learning in bounded domains with worst-case adversaries} and traditional \emph{online transductive learning}. For worst-case adversaries in binary classification with domain size $\abs{ \cX }$, we achieve a} regret of $ O( \sqrt{T (d \abs{ \cX} )^{1/2}}  ) $, improving upon the $ O(\sqrt{T \abs{\cX} }) $ bound of \cite{daskalakis2016learning}.
{For online transductive binary classification we achieve a regret bound of} $ O( T^{3/4} d^{1/4} ) $, improving upon the $ O( T^{3/4} d^{1/2}  )  $ bound in \cite{kakade2006batch}. {Both improvements are enabled by our novel Poissonized FTPL analysis and techniques which achieve stronger regret bounds in the binary classification setting.}

\subsection{Technical Overview}

    \paragraph{Relaxations for Beyond Worst-Case Adversaries.}
    The design of our algorithms is based on the random playout technique and we will use 
the \emph{admissible relaxation}  framework of \cite{rakhlin2012relax} to analyze these algorithms. 
    Each of our algorithms generates and randomly labels instances as a stand-in for the future. We show how this framework is especially useful for analyzing online learning algorithms in the beyond worst-case setting.
    In particular, the $K$-hint transductive setting, in contrast to the classical transductive setting, includes hints that are never realized as part of the sequence.
     To show that these hints have a small impact on the achievable regret, we prove that the \emph{regularized Rademacher complexity} is monotone in the set of provided hints (\Cref{lem:monotonicity}).
    This monotonicity allows us to leverage existing relaxations that are admissible for highly predictable adversaries and to show that the regret gracefully degrades as a function of the power of the adversary.

    \paragraph{Self Generation of Hints.}
In the smoothed analysis setting, the algorithm does not have access to hints, nevertheless, smoothness captures a level of predictability about the future. 
This is captured by a technique from  \cite{haghtalab2022smoothed}
    that shows that any sequence of $T$ instances generated by adaptive smoothed adversaries can be seen being a subset of set of $T/\sigma$ uniformly random instances from $\cX$  with high probability.  We use this to self-generate hints and draw a parallel between $K$-hint transductive learning and smoothed analysis.
    Though the next instance given by the adversary is not guaranteed to be in this hint set, we show that this process accounts for the uncertainty in each step.

    Another key aspect is the distribution from which hints are self-generated.
    We will crucially use the Poissonization technique
    in which we generate the number of hints from an appropriately chosen Poisson distribution.
    This allows us an additional degree of independence that is essential to controlling the loss from one step to the next.
   Both of these properties are key in our analyses of the admissibility of the corresponding relaxation.

    \paragraph{Generalization, Stability, and Admissibility.}
    In addition, we make use of a connection between the traditional notion of algorithmic stability and admissibility.   
    Stability is a well-studied notion capturing how little the distribution of the actions of the learner changes across time steps and is used in the analyses of algorithms such as follow-the-perturbed-leader and follow-the-regularized-leader.
    We show that the admissibility of the relaxation in the smoothed online setting is implied by the stability. This allows us to use information theoretic techniques that crucially exploit 
    the independence provided by Poissonization as discussed earlier. 

    {As we show, the stability analysis of the algorithm also depends crucially on the generalization error of the ERM output when it is trained on uniformly self-generated hints and 
    tested on smoothly distributed fresh instances.}
    In order to bound the generalization error, we take advantage of a stronger property implied by the coupling lemma from \cite{haghtalab2022smoothed}.
    This states that there exists a coupling between uniform and adaptive smooth processes, such that when the inclusion property is satisfied, the distribution of the uniform variables realized in the inclusion \emph{conditional on the unrealized uniform variables} is also identical to the smooth distributions given by the adversary. This will be instrumental for bounding the generalization error by allowing us to  extract smooth variables from a set of uniform variables, which can then be used to for the purpose of symmetrization. This result may be of independent interest.

\subsection{Related works}

Our work relates to several paradigms and approaches to online learnability.

\paragraph{Oracle-Efficient Online Learning.}
Since the seminal work of \cite{FTPL}, inspired by application domains such as game theory, there has been a long line of work elucidating the computational aspects of online learning. 
\cite{FTPL} proposed the influential follow-the-perturbed-leader algorithm.
\cite{KakadeKalaiApproximation} study notions of regret when the learner is given access to an approximate optimization oracle.
\cite{kakade2006batch} study the transductive learning setting and give an efficient algorithm that converts offline learnablity to online learnability.
\cite{rakhlin2012relax} propose a general \emph{admissible relaxation} framework to develop efficient algorithms based on the upper bound of the value of the game. 
This framework has also been built on and extended in a variety of settings such as auctions~\cite{Oracle_Efficient}, contextual learning~\cite{syrgkanis2016efficient,foster2020beyond,simchi2021bypassing}, and reinforcement learning~\cite{foster2021statistical}.
On the flip size, \cite{HazanKoren} show that an $\Omega(\sqrt{N})$ lower bound is unaviodable in general in order to obtain nontrivial regret where the $N$ is the number of actions of the learner suggesting that one needs to look beyond the worst-case in order to get truly efficient algorithms.

\paragraph{Beyond Worst-case Approaches to Online Learning.}

Various notion of beyond worst-case behavior of online learning has been studied in the literature.
\cite{NIPS2011_4262} consider online learning where the adversary is constrained and build a framework based on minimax analysis and constrained sequential Rademacher complexity to analyze regret in these scenarios.
These techniques have been applied to other constrained settings \cite{activelearning}. 

\cite{Gupta_Roughgarden} consider smoothed online learning when looking at problems in online algorithm design.  
They prove that while optimizing parameterized greedy heuristics for combinatorial problems
in presence of smoothing this problem can be learned with non-trivial sublinear regret.
\cite{Cohen-Addad} consider the same problem with an emphasis on the per-step runtime being logarithmic in $T$. 
\cite{haghtalab2020smoothed,haghtalab2022smoothed} both study the notion of smoothed analysis with adaptive adversary and show that statistically the regret is bounded by $ O ( \sqrt{T d \log ( 1 / \sigma ) } ) $.

Smoothed analysis has also been studied in a number of other online settings, including linear contextual bandits, \cite{kannan2018smoothed,raghavan2018externalities}, but the focus has been on achieving improved regret bounds for the greedy algorithm that is not no-regret in the worst-case.

Another line of work has focused on the future sequences being predictable given the past instances. 
\cite{PredictableSequences} incorporate additional information available in terms of an estimator for future instances.
They achieve regret bounds depending on the path length of these estimators and can  beat the worst-case $\Omega(\sqrt{T})$ if the estimators are accurate. 
\cite{hazan2007online} models predictability as knowing the first coordinate of loss vectors, which is revealed to the learner before he chooses actions.
Some other work model predictability through hints which are additive estimate of loss vectors \cite{hazan2010extracting,rakhlin2013online,steinhardt2014adaptivity,mohri2016accelerating}. 
\cite{dekel2017online,hints} considers settings where the learner has access to hints in form of vectors that are weakly correlated with the future instances and show exponential improvement in the regret in some cases. The literature on hints represents an active and growing subarea of online learning (see \cite{hints} and references within).

\paragraph{Concurrent Work.}
In a concurrent and independent work, \cite{BlockDGR22} also gives oracle-efficient algorithms for smoothed online learning. 
\cite{BlockDGR22} obtains a regret bound of
$\widetilde{O}(\sqrt{Td \sigma^{-1}})$.
In comparison, our main result (Theorem~\ref{thm:FTPL}) concerning smoothed online learning demonstrates a regret bound of 
$\widetilde{O}(\sqrt{Td \sigma^{-1/2}})$ with strictly better dependence on $\sigma$.
Our regret bound's improved dependence on parameter $\sigma$ can be attributed the following technical innovations of this work:  1) The relationship between stability and relaxation-based methods,
2) the careful analysis of \emph{modified generalization error} and \emph{stability} via a new coupling-based argument, and 3) a Poissonization approach for \emph{self-generating} hints that allow us to leverage information theoretic arguments. {Importantly, the novel techniques for achieving improved dependence on $\sigma$ also enable us to improve on the small-domain result of \cite{daskalakis2016learning} and the transductive learning result of \cite{kakade2006batch}.}

Interestingly, both our work and \cite{BlockDGR22} use a relaxation-based algorithms as a warmup achieving regret bounds of $\widetilde{O}(\sqrt{Td \sigma^{-1}})$ (our Theorem~\ref{thm:regret-real}) and $\widetilde{O}(\sigma^{-1}\sqrt{Td})$ \cite[Theorem 8]{BlockDGR22} for learning real-valued functions.
For the warmup, our regret bound's improved dependence on parameter $\sigma$ is due to our different approach to self-generating hints that allow us to leverage the monotonicity of the Rademacher complexity.

\vspace{-6pt}
\section{Preliminaries}
Let $\cX$ be the space of instances, $\cY=[-1,1]$ be the space of labels, 
and $\cH:\cX\to\cY$ be the hypothesis class with pseudo dimension $d$ (See \cref{def:pseudo-dimension} or \cite{anthony1999neural} for the definition of pseudo dimension). Let $l:\cY\times\cY\to[0,1]$ be a convex loss function with Lipschitz constant $G$ in its first component.

In online learning with adaptive adversaries, 
the learner and the adversary plays a repeated game for $T$ time steps.
At each step $t\in[T]$,
the adversary chooses a distribution $\cD_t^{\cX}\in\Delta(\cX)$. A random instance $x_t\sim \cD_t^{\cX}$ is then drawn and presented to the learner. After receiving $x_t$, the learner predicts its label to be $\hy_t\in\cY$, while the adversary simultaneously chooses $y_t\in\cY$ as its true label. The learner then suffers loss $l(\hy_t,y_t)$.
This is equivalent to a setting where the adversary chooses a distribution $\cD_t\in\Delta(\cX\times\cY)$ over labeled instances $s_t=(x_t,y_t)$ and the learner simultaneously chooses a classifier $h_t\in\cY^{\cX}$. 
We will abbreviate $\cD_t^{\cX}$ to $\cD_t$ when it is clear from the context.

We allow the adversary to be adaptive, i.e., it can choose each $\cD_t$ based on the realizations of the inputs as well as the learner's predictions in previous time steps. Let $\sD$ denote the adaptive sequence of distributions $\cD_1,\cdots,\cD_T$. 
Accordingly, let $\cQ_t\in\Delta(\cY)$ denote the learner's prediction rule on instance $x_t$, and let $\sQ$ denote the adaptive sequence of distributions $\cQ_1,\cdots,\cQ_T$.
We denote the expected regret of a learner with prediction rules $\sQ$ on the adaptive sequence $\sD$ by
\begin{align}
    \Ex[\regret(T, \sD, \sQ)]=\Ex_{\sD,\sQ}\left[\sum_{t=1}^T l(\hy_t,y_t)-\inf_{h\in{\cH}}\sum_{t=1}^T l(h({x_t}),y_t)\right]. \label{eq:regret}
\end{align}
We remove $\sD$ and $\sQ$ from this notation when they are clear from the context.

\paragraph{Offline Optimization Oracle}
We consider computationally efficient algorithms given access to an offline optimization oracle. 
For the case of binary classification, the oracle outputs the solution of empirical risk minimization on the input data.

\begin{definition}[ERM Oracle]
    \label{def:ERM}
    For a hypothesis class $\cH$ and a loss function $l$, the oracle $\OPT$ ($\opt$) takes a set~\footnote{The inputs to the oracle are multisets. Unless specified otherwise, all the sets in this paper refer to multisets.} of inputs {$S = \{ (x_i,y_i) \}_{i\in [I]}$ where  $ (x_i,y_i)\in\cX\times\cY$ for all $i\in[I]$} and returns
    \vspace{-4pt}
\begin{align}
    \OPT_{\cH,l}({S}) = \inf_{h\in\cH}\sum_{i=1}^{I} l(h(x_i),y_i) ~\text{ and }~
    \opt_{\cH,l}({S}) \in \arg\inf_{h\in\cH}\sum_{i=1}^{I} l(h(x_i),y_i).
\end{align}
\end{definition}

For the case of real-valued functions, we consider an oracle that can minimize a mixture of binary and real-valued loss values defined below.

\begin{definition}[Real-valued optimization oracle]
    \label{def:real-valued-oracle}
  For a hypothesis class $\cH$ and two loss functions $l^r$ and $l^b$, the oracle $\OPT$ takes two sets of inputs $S$ and $S'$ over $\cX\times \cY$ and returns
$$
    \OPT_{\cH,l^{\mathrm{r}}, l^{\mathrm{b}}}(S; S') = \inf_{h\in\cH}\Big(\sum_{(x,y)\in S} l^{\mathrm{r}}(h(x),y) + \sum_{(x',y')\in S'} l^{\mathrm{b}}(h(x'), y')\Big).
 $$ 
\end{definition}
\vspace{-8pt}

We remark that these oracles are used in most previous works, including~\cite{rakhlin2012relax}. They constitute a special form of regularized loss minimization oracles, where the regularization is given directly by a random process. 
For the binary setting where $\cY=\{\pm1\}$ and $l^{\mathrm{r}}=l^{\mathrm{b}}=\indicator{\hy\neq y}$, the above optimization oracle is equivalent to ERM oracles.

We consider each call to the offline optimization oracle as having unit cost plus the additional runtime needed for creating and inputting the set of inputs that is linear in the length of the said histories. We note that our approach and results directly extend to using ERM oracles with (arbitrarily small) additive approximation error, such as those guaranteed by FPTAS optimization algorithms, using standard techniques presented by \cite[Section 6]{Oracle_Efficient}.

\begin{remark}
    Though the oracles as defined above are required to work on arbitrary inputs, both \cref{alg:real-valued,alg:FTPL} from our work only call the optimization oracle on instances sampled from the smoothed distributions or the uniform distribution.
    Thus, it suffices to have oracles that work for average-case instances.
    This makes the design of such oracles easier both from theoretical and practical points of view.
\end{remark}

\subsection{Smoothed Online Learning}
We work with the smoothed adaptive online adversarial setting from \cite{haghtalab2022smoothed}. 
We will consider $\sigma$-smooth adversaries, where a distribution is  $\sigma$-smooth if its density is upper bounded by $1/\sigma$ times the density of the uniform distribution over the same domain. We remark that all of our results generalize to arbitrary known base distributions as well.

\begin{definition}[$\sigma$-smoothness]
    Let $\cX$ be a domain that supports a uniform distribution $\unif$. A measure $\mu$ on $\cX$ is $\sigma$-smooth if for all measurable subsets $A\subseteq \cX$, $\mu(A)\le\frac{\unif(A)}{\sigma}$. The set of all $\sigma$-smooth distributions on domain $\cX$ is denoted by $\Delta_\sigma(\cX)$. 
\end{definition}

In online learning with adaptive smoothed adversaries, at each time step $t$, the adversary chooses a distribution $\cD_t$ whose marginal on $\cX$ is $\sigma$-smooth. The choice of 
$\cD_t$ can depend the previous instances $\{(x_i,y_i)\}_{i=1}^{t-1}$ as well as the learner's previous predictions. We denote with $\sD_\sigma$ the adaptive sequence of $\sigma$-smooth distributions on the instances.
The corresponding definition of regret in this setting is given by \cref{eq:regret} with $ \cD_{\sigma} $ as the set of distribution for the adversary.

An important property of smoothness is that it implies coupling between uniform and adaptive smooth processes. 
We will consider the original result from \cite{haghtalab2022smoothed} in \Cref{lemma:coupling} and a slightly strengthened statement in \Cref{lemma:coupling_strong}. 
In \cref{sec:binary}, we will provide new insights on the properties of such couplings.

\subsection{Transductive Online Learning with $K$ Hints} \label{sec:transductive-online-learning-k-hints}
In the traditional transductive setting, the adversary releases the sequence of unlabeled instances $\{x_t\}_{t=1}^T$ to the learner before the game starts. 
{We generalize this setting and introduce a $K$-hint version of transductive learning. In this setting, the exact sequence of instances is replaced with a sequence of $K$ \emph{hints} per time step such that the set of hints at each time step includes the instance at that time step.}
More formally, before the interaction starts, the adversary releases $T$ sets (multisets) of size $K$ to the learner. We denote these sets by $\{Z_t=\{z_{t,1},\cdots,z_{t,K}\}\}_{t=1}^T$. On releasing these sets, {the adversary promises to always pick $\cD_t$ supported only on the elements of $Z_t$.} 
The regret is defined by \cref{eq:regret} with the appropriate restriction on the adversary.

 \subsection{Relaxations and Admissibility}
 \label{sec:admissibility}

 Our algorithms for the general (real-valued) setting relies on the \emph{admissible relaxation} framework proposed in \cite{rakhlin2012relax}. 
A relaxation $\rel_T$ is a sequence of functions $\rel_T(\cH|s_{1:t})$ for each $t\in[T]$, which map the history of the play to real values that upper bounds the conditional value of the game. 
We will make use of an important algorithmic aspect of the relaxation framework, which states that
whenever an algorithm is \emph{admissible} with respect to some relaxation, its expected regret can be upper bounded in terms of the value of the relaxation at the beginning of the game.

\begin{definition}[Admissibility]
\label[definition]{def:admissibility}
In an online learning setting where the adversary is restricted to playing $\cD_t\in\mathfrak{D}_t\subseteq\Delta(\cX)$ at each time $t$, let $\sQ$ be an algorithm that gives rise to a sequence of distributions $\cQ_1,\cdots,\cQ_T$ on the predicted labels.
    We say $\sQ$ is admissible with respect to a relaxation $\{\emph{\rel}_T(\cH\mid s_{1:t})\}_{t=0}^T$, if for any sequence of instances $s_{1:T}$,
    \begin{enumerate}[noitemsep,topsep=0pt,itemsep=-1ex,partopsep=1ex,parsep=1ex]
    \item For all $t\in[T]$,
    \vspace{-8pt}
    \begin{align}
        \sup_{\cD_t\in\mathfrak{D}_t}\Ex_{x_t\sim \cD_t}\sup_{y_t\in\cY}\!\left\{\!
        \Ex_{\hy_t\sim\cQ_t}\![l(\hy_t,y_t)]\!+\!\emph{\rel}_T(\cH\mid s_{1:t-1}\!\cup\!(x_t,y_t))\!
        \right\}
        \!\le\! \emph{\rel}_{T}(\cH\mid s_{1:t-1});
    \end{align}
    \item The final value satisfies $\emph{\rel}_T(\cH\mid s_{1:T})\ge-\inf_{h\in\cH}L(h,s_{1:T}).$
    \end{enumerate}

    \end{definition}

    In the smoothed online learning setting, $\mathfrak{D}_t=\Delta_\sigma(\cX)$ is the set of $\sigma$-smooth distributions on $\cX$. In the $K$-hint transductive online learning setting, $\mathfrak{D}_t$ is the set of distributions supported on $Z_t$.

    The following proposition is the analog of the results of \cite{rakhlin2012relax} when the adversary is smooth.
    The full proof is presented in \cref{appendix:admissibility}.

    \begin{proposition}[Regret Bound via Admissibility]
    \label[prop]{thm:admissibility}
    In the online learning setting where the adversary's choice is restricted to $\mathfrak{D}_t$ ($t\in[T]$), let $\sQ=(\cQ_1,\cdots,\cQ_T)$ be an algorithm that is admissible with respect to relaxations $\emph{\rel}_T(\cH)$,
    then the following bound on the expected regret holds regardless of the strategies $\sD$ of the adversary:
    \begin{align}
        \Ex[\regret(T,\sQ,\sD)]
\le \emph{\rel}_T(\cH\mid \emptyset)+O(\sqrt{T}).
    \end{align}
    \end{proposition}

\subsection{Follow the Perturbed Leader} \label{sec:ftpl}

When the labels are binary, we design am algorithm which achieves improved regret bounds using the Follow the Perturbed Leader (FTPL) principle~\cite{FTPL}. 
An FTPL algorithm makes predictions by applying ERM oracle to the perturbed histories of the play. At every time step $t\in[T]$, the algorithm chooses a distribution over labeled instances, from which it draws $N$ random instances $(\widetilde{x}_1^{(t)},\widetilde{y}_1^{(t)}), \cdots, (\widetilde{x}_N^{(t)}, \widetilde{y}_N^{(t)})$. The predicted label is then given by $\hy_t=h_t(x_t)$, where
\begin{align}
    h_t \gets \opt_{\cH,l}\left(s_{1:t-1}\cup \{(\widetilde{x}_i^{(t)}, \widetilde{y}_i^{(t)})\}_{i\in [N]}\right).
\end{align}
The standard analysis of FTPL bounds the expected regret as follows:
\begin{align}
    \Ex[\regret]\le \underbrace{\Ex\!\left[\sum_{t=1}^T
    l\!\left(h_t(x_t),y_t\right)-l\!\left(h_{t\!+\!1}(x_t),y_t\right)
    \right]}_{\text{Stability}}
    +\underbrace{\Ex\!\left[\sup_{h\in\cH}
    \sum_{i=1}^N l(h(\widetilde{x}_i), \widetilde{y}_i)-\sum_{i=1}^N l(h^*(\widetilde{x}_i), \widetilde{y}_i)
    \right]}_{\text{Perturbation}},
\end{align}
where $h^*=\arg\inf_{h\in\cH}\sum_{t=1}^T l(h(x_t),y_t)$.

Note that the perturbation term is already well-understood from statistical learning theory since it is essentially the Rademacher complexity of $\cH$ for sample size $N$. 
Therefore, we will focus on bounding the stability term by designing perturbations that can leverage the anti-concentration property of smoothed adversaries.
 \section{$K$-hint Transductive Learning}
\label{sec:warmup}

As a prelude, we will first look at the $K$-hint transductive learning with real-valued labels, which 
highlights some challenges that are present in the smoothed online setting while allowing us to discuss the required tools in a simpler scenario. 
The statistical upper bound given by the inefficient algorithm that simply plays Hedge on a covering of the projection of $\cH$ to the hint set is $ O ( \sqrt{d T \log (TK)  } ) $, as shown in \Cref{thm:transductive-statistical}.
We will show an oracle-efficient regret upper bound of $O( \sqrt{d T K  } )$ by constructing an oracle-efficient algorithm based on the random playout technique. We consider the optimization oracle defined in \Cref{def:real-valued-oracle} with the loss functions specified by $l^{\mathrm{r}}(\hat y,y)=\frac{1}{2G}l(\hat y,y)$ and
$l^{\mathrm{b}}(\hat y,y)=\indicator{\hy\neq y}-\frac{1}{2}$.

Let us begin by describing our algorithm for the setting of $K$-hint transductive learning. 
At each time step $t$, our algorithm applies the offline optimization oracle to two input sequences: One where the real history $s_{1:t-1}$ is mixed with two copies\footnote{We use two copies to scale the loss appropriately.} of randomly labeled set of all hints corresponding to future time steps and the current instance is labled $+1$, and another, where the current instance is labeled $-1$. 

More specifically, with $\cE^{(t)}=\{\epsilon_{i,k}^{(t)}\}_{{i=t+1:T},{k=1:K}}$ denoting the set of random labels and $S^{(t)}=(Z_{t+1:T},\cE^{(t)})$ denoting the set of hints labeled by $\cE^{(t)}$, we consider
\begin{align}
	\widehat{y_t}=& \OPT\left( s_{1:t-1}; S^{(t)}\cup S^{(t)}\cup\{(x_t, -1)\} \right) - 
	\OPT\left( s_{1:t-1}; S^{(t)}\cup S^{(t)}\cup\{(x_t, +1)\} \right).
	\label{eq:prediction-rule-real}
\end{align}

Since the two input sequences to the optimization oracle only disagree on one label, the difference in the optimal errors is always bounded within $[-1,+1]$, thus guarantees $\hy_t\in\cY$.
{Intuitively, the reason $\hat y_t$ includes the gap between the error of these two optimal classifiers is to make the algorithm hedge its bets against how the adversary is going to label the current instance $x_t$.} A formal description of the algorithm is given in \Cref{alg:transductive_k_hints}.  
The following theorem provides an upper bound on the regret.

\begin{theorem}[Regret Bound for Efficient $K$-Hint Transductive Learning]
    \label{thm:regret-transductive}
    In the setting of transductive learning with $K$-hints, the above algorithm has  expected regret bound of 
$O(\sqrt{dTK}).$
    The algorithm can be implemented using two calls to the optimization oracle per round. 
    \end{theorem}

The proof of  \cref{thm:regret-transductive} follows a similar approach to the work of \cite{rakhlin2012relax} and uses the admissible relaxation framework. In this section, we provide an overview of this approach and the modifications that allow us to incorporate a larger set of hints. 
Specifically, we will show that the algorithm is admissible with respect to the following relaxation:
 \begin{align}
 \rel_T(\cH\mid s_{1:t})=\Ex_{\cE^{(t)}}\left[
            \sup_{h\in\cH}\left\{
                2G\sum_{i=t+1:T\atop k=1:K}
            \epsilon_{i,k}^{(t)} h(z_{i,k})
            -\sum_{i=1}^t l(h(x_i),y_i))
            \right\}
    \right],\qquad t=0,\cdots,T.
     \label{eq:relaxation-transductive}
 \end{align}
 
 To simplify the notation, for a set of unlabeled instances $Z=\{z_i\}_{i=1}^I$ and a function $\Phi:\cH\to\mathbb{R}$, we define $\mathfrak{R}(\Phi,Z)$ as the Rademacher complexity for set $Z$ regularized by $\Phi$, that is
$$
     \mathfrak{R}({\Phi},Z)= \Ex_{\epsilon_{1:I}\overset{\text{iid}}{\sim}\unif(\pm1)}\Big[
        \sup_{h\in\cH}\Big\{\sum_{i\le I}\epsilon_i h(z_i)
        +\Phi(h)\Big\}
        \Big].
        $$
 Then the relaxation at the end of time step $t$ can be written as the Rademacher complexity for the union of future hints, regularized by the past total loss. That is,
 \begin{align}
    \rel_T(\cH\mid s_{1:t})=2G\cdot\mathfrak{R}(- L^{\mathrm{r}}( \cdot , s_{1:t} ) , Z_{t+1:T}),
    \label{eq:relaxation-rademacher}
\end{align}
where $L^{\mathrm{r}}(h,s_{1:t})=\sum_{i=1}^t l^{\mathrm{r}}(h(x_i),y_i)=\frac{1}{2G}\sum_{i=1}^t l(h(x_i),y_i)$ for $h\in\cH$.

 We remark that while the original definition of relaxation keeps track of the sequence of inputs that are already known at the end of each time step ($s_{1:t}$), the regularized Rademacher complexity also explicitly emphasizes the instances $Z_{t+1:T}$ denoting unknown future. 
 Between successive time steps, one extra data point $s_t$ will be added to the set of known inputs, but the set of unknown future instances will shrink by $K$.
 
To use the relaxation framework and \Cref{thm:admissibility}, it suffices to establish two claims: 1) the relaxation in \cref{eq:relaxation-rademacher} is admissible in the $K$-hint setting, 2) the value of this relaxation at the beginning of the game is not too large.

For the second claim, we notice that $\rel_T(\cH\mid \emptyset)$ is equal to the unregularized Rademacher complexity for the dataset that includes all the hints. Since there are at most $TK$ hints, the Rademacher complexity is at most $\widetilde{O}(\sqrt{d TK })$ according to \Cref{lemma:rademacher-bound}. That's where we get the extra $\sqrt{K}$ in the bounds compared to the standard transductive setting.

 The first claim is the more technically interesting one.
 For admissibility, here we will focus on proving the following bound
 \begin{align}
     \sup_{x_t\in Z_t}\!
     \underbrace{\sup_{y_t\in\cY}\!\big\{\!
        \Ex_{\hy_t\sim\cQ_t}[l^{\mathrm{r}}(\hy_t,y_t)]\!+\!\mathfrak{R}(\!-\!L_t^{\mathrm{r}},Z_{t+1:T}\!)\!
        \big\}}_{\text{(a)}}\!\le\! 
        \sup_{x_t\in Z_t}\!\mathfrak{R}(\!-\!L_{t-1}^{\mathrm{r}},Z_{t+1:T}\!\cup\!\{\!x_t\!\}\!) \!\le\!
    \mathfrak{R}(\!-\!L_{t-1}^{\mathrm{r}},Z_{t:T}),
    \label[ineq]{aaaaa}
 \end{align}
 where $L_t^{\mathrm{r}}(h)$ abbreviates for $L^{\mathrm{r}}(h,s_{1:t})$.

 Let us consider the L.H.S of the above inequality and note that for any fixed $x_t$, the term (a) captures the standard transductive learning setting with $Z_{t+1:T}$ being the set of unlabeled instances for the future. 
 In this case, the convexity of loss function $l^{\mathrm{r}}$ together with the min-max theorem can be used to show that the learner's strategy $\cQ_t$, which makes the two values inside the supremum over $\cY$ equalize as $y_t$ takes value $-1$ and $+1$, is indeed the optimal strategy. At a high level, this technique which is also used by \cite{rakhlin2012relax}, gives rise to the algorithm in \cref{eq:prediction-rule-real} and proves the first inequality in \cref{aaaaa}. We refer the readers to \cite[Lemma 12]{rakhlin2012relax} for more details about the proof.
 
   The second transition in \Cref{aaaaa} can be establishing the fact that  regularized Rademacher complexity is monotone in the dataset.    See \cref{appendix:monotonicity} for a proof of \Cref{lem:monotonicity}.

  \begin{lemma}[Monotonicity of Regularized Rademacher Complexity]
    \label{lem:monotonicity}
    For any dataset $z_{1:m}\in\cX^m$ and any additional data point $x\in \cX$, we have $        \mathfrak{R}(\Phi,z_{1:m})\le\mathfrak{R}(\Phi,z_{1:m}\cup\{x\}).$
    \end{lemma}
   
   This monotonicity can be used recursively to add the extra set of hints $Z_t$ to the relaxation. This
   implies $\sup_{x_t\in Z_t}\mathfrak{R}(-L_{t-1},Z_{t+1:T}\cup\{x_t\})\le \mathfrak{R}(-L_{t-1},Z_{t+1:T}\cup Z_t)$. 
\section{Oracle-Efficient Smoothed Online Learning}
\label{sec:binary}
{
In this section, we prove our main result for the case of smoothed analysis of online learning
and give the first oracle-efficient algorithms whose regret are sublinear in $T$ and pseudo (or VC) dimension.
\begin{theorem}[Regret Bound for Smoothed Online Regression]
	\label{thm:regret-real}
	For any $\sigma$-smooth adversary $\sD_\sigma$, \Cref{alg:real-valued} has expected regret upper bounded by  $ \widetilde{O}(G\sqrt{Td/\sigma})$, where $\widetilde{O}$ hide factors that are polynomial in $\log(T)$ and $\log(1/\sigma)$.
	Here $ G$ is the Lipschitz constant of the loss and $d$ is the pseudodimension of the class. 
	Furthermore, the algorithm is oracle-efficient: at every round $t$, this algorithm uses two oracle calls with histories of length $\widetilde{O}(T/\sigma)$.
\end{theorem}

\begin{theorem}[Regret Bound for Smoothed Online Classification]\label{thm:FTPL}
In the setting of online binary classification with $\sigma$-smoothed adversaries, \cref{alg:FTPL} has regret that is at most
\begin{align}
    \widetilde{O}\left( \min\left\{ \sqrt{ \frac{Td}{\sigma^{1/2}} }, \sqrt{T(d|\cX|)^{1/2}} \right\}\right ). 
\end{align}
 Furthermore, \cref{alg:FTPL} is a \emph{proper} learning oracle-efficient: at every round $t$, this algorithm uses a single ERM oracle call a history that is of length $t + O(T/\sqrt{\sigma})$ with high probability.
\end{theorem}

{We note that \Cref{thm:regret-real} and \Cref{thm:FTPL} establishes that under smoothed  analysis online learning is as computationally efficient as offline learning. This is in contrast with the results of \cite{HazanKoren} that showed a computational separation between offline learning and online learning with worst-case adversaries. }
It is worth noting that \Cref{thm:FTPL} also matches or improves upon several existing results for worst-case online learning (without smoothed analysis). For example, for finite domains where worst-case adversaries are vacuously $\sigma$-smooth for $1/\sigma = |\cX|$, \Cref{thm:FTPL} gives an oracle-efficient regret bound of
$O(T^{1/2} (d |\cX|)^{1/4})$ which is a refinement of \cite{daskalakis2016learning} regret bound of
$O(\sqrt{T|\cX|})$ because $d\leq |\cX|$.
\Cref{thm:FTPL} also implies the regret bound of $\widetilde{O}(T^{3/4}d^{1/4})$ for the classical transductive learning setting, which corresponds to finite domains with $\sigma = 1/T$ parameter. This improves the $\widetilde{O}(T^{3/4}d^{1/2})$ bound of \cite{kakade2006batch}. See Corollary \ref{cor:smalldomain} for a precise statement. 

\begin{corollary}[Regret for Small Domain] \label{cor:smalldomain}
    There is an oracle-efficient algorithm for online learning with binary labels (in the worst-case) that achieves a regret of   $ O (\sqrt{T(d|\cX|)^{1/2}} )$ for any hypothesis class with VC dimension $d$ on domain $ \cX $.
    For transductive learning with binary labels, there is an oracle efficient algorithm, with regret $ O \left( T^{3/4} d^{1/4}  \right)  $. 
\end{corollary}

\begin{remark}
    {Both \cref{alg:real-valued,alg:FTPL} can be adapted to deal with unknown $\sigma$. In particular, the regret bounds hold for any approximation $\widetilde{\sigma}$ that is a lower bound of the real $\sigma$ up to constant multiplicative factors. This corresponds to settings where the world is more smooth than we give it credit. Even when we have extremely poor upper and lower bounds, we can use hedging to still get non-trivial regret with only a minor blow up in computation. 
    We will provide more details in \cref{appendix:unknown_sigma} about working with knowledge of approximate $\widetilde{\sigma}$.}
    \end{remark}

The key challenge with adapting our previous approach is that, unlike the $K$-hint transductive setting, in smoothed analysis the learner no longer has access to any hints. For finite domains, such as the above examples, our (and existing) algorithms self-generate hints from every instance in $\cX$, using uniform or geometric random variables. 
But these choices do not directly extend to large and infinite domains where most of the domain $\cX$ will not be present in the hints.
Therefore, for smoothed analysis with large or infinite domains, our algorithms must self-generate hints that leverage the \emph{anti-concentration} properties of smooth distributions.
In \Cref{sec:older-algorithm}, we show how this can be done by taking $K = \widetilde{O}(1/\sigma)$ uniformly random instances from $\cX$ and using the techniques we developed in \Cref{sec:warmup}
to achieve a regret bound of $\widetilde{O}(\sqrt{Td/\sigma})$.
Interestingly, the particular form of randomness used for producing hints is consequential. In \Cref{sec:improved}, we show that in the special case of binary classification, 
the Poisson distribution over the size of the hint set is an appropriate choice to leverage properties implied by the smoothness of the distributions, while leading to relaxations that enjoy better admissibility and stability guarantees. 
}

{Another important aspect of stability is the use of a modified definition of generalization error, i.e., relating the performance of 
ERM trained with uniform self-generated hints had we re-sampled the last smoothed adversarial instance. We will provide insights about the coupling approach that allows us to relate  generalizability with respect to smooth distributions to that of the uniform distribution  directly. }

\subsection{Learning with Real-Valued Functions}
\label{sec:older-algorithm}

In this section, we show how insights and techniques for the $K$-hint transductive setting in \cref{sec:warmup} transition to the setting of  smoothed online setting and lead to $\widetilde{O}(\sqrt{Td/\sigma})$ regret. When the labels are binary, we will further build upon this connection in \Cref{sec:improved} to improve on the regret bound of this section.

The key challenge with adapting the previous techniques is to self-generate hints that leverage  the anti-concentration properties of smooth distributions. 
In particular, we will use the coupling technique introduced by \cite{haghtalab2022smoothed} (see \Cref{lemma:coupling} for a complete description) to replace the sequence of $T$ random inputs $x_{1:T}$ generated by the adaptive adversary with $TK$ inputs $\{z_{t,k}\}_{t=1:T, k=1:K}$ that are generated i.i.d. from the uniform distribution over $\cX$, such that with high probability
    $\{x_1,\cdots,x_T\}\subseteq  \{z_{t,k}\}_{{t\in[T]},{k\in[K]}}$.
This property highlights the connection between a smoothed adaptive adversary and a transductive adversary with $K$ uniform random hints.

From a computational point of view, this observation implies that a learner can generate uniform random hints, denoted by $\{v_{i,k}^{(t)}\}_{{i=t+1:T},{k=1:K}}$, and use them in place of $\{z_{i,k}\}_{{i=t+1:T},{k=1:K}}$ in \cref{alg:transductive_k_hints}. The resulting algorithm is summarized in \cref{alg:real-valued} and leads to the  regret upper bound in \Cref{thm:regret-real}.

\begin{algorithm}
	\caption{Oracle-Efficient Smoothed Online Learning for Real-valued Functions}
	\label{alg:real-valued}
	\DontPrintSemicolon
	\LinesNumbered
	\KwIn{$T,\sigma$}
	$K\gets {100\log T}/{\sigma}$.\label{alg:def-K}\;
	\For{$t\leftarrow 1$ \KwTo $T$}{
		Receive $x_t$.\;
		\For{$i=t\!+\!1,\cdots,T$; $k=1,\cdots,K$}{
			Draw new $v_{i,k}^{(t)}\sim\unif(\cX)$.\;
			Draw new $\epsilon_{i,k}^{(t)}\sim\unif(\{-1,+1\})$.
		}
		$S^{(t)}\gets \left\{ (v_{i,k}^{(t)},\epsilon_{i,k}^{(t)})
		\right\}_{\myatop{i=t+1:T}{k=1:K}}$.\;
		$\hy_t\gets\OPT\left( s_{1:t-1}; S^{(t)}\cup S^{(t)}\cup\{(x_t, -1)\} \right) \!-\! 
        \OPT\left( s_{1:t-1}; S^{(t)}\cup S^{(t)}\cup\{(x_t, +1)\} \right).$\;
		Receive $y_t$, suffer loss $l(\hy_t,y_t)$.
	}
\end{algorithm}

To prove \cref{thm:regret-real}, we adjust the relaxation we used in \Cref{sec:warmup} and show that it is admissible with respect to the algorithm. We use the following relaxation in the smoothed learning setting,
\begin{align}  \label{eq:relaxation-rademacher}
    \rel_T(\cH\mid s_{1:t})=2G\Ex_{V^{(t)}\simiid \unif(\cX)}\left[\mathfrak{R}(-{L^{\mathrm{r}}(\cdot,s_{1:t})},V^{(t)})
        \right]+2G\beta(T-t),
 \end{align}
where $K = 100\log T / \sigma $ and $\beta=10TK(1-\sigma)^K$. 
This relaxation incorporates two modifications to the relaxation in \cref{eq:relaxation-transductive}. 
The first is an expectation over the randomness of future hints. {A subtle point here is that this randomness crucially matches both random $V^{(t)}$ that the algorithm has access to and the coupled variables $\{z_{i,k}\}$ at a distribution level that are never revealed to the algorithm. To ensure that our approach works with adaptive adversaries, it is essential that $V^{(t)}$s are fresh samples per round. }
The second is an additional time dependent term $\beta(T-t)$, which accounts for the total failure probability of the coupling argument in the future $T-t$ time steps. 
Since the coupling argument shows that $K={O}(\log(T)/\sigma)$, we use parameter $\beta=o(T)$ and achieve a regret upper bound of $\widetilde{O}(\sqrt{dTK})=\widetilde{O}(\sqrt{Td/\sigma})$. See \cref{appendix:omit-older-proof} for more details about the proof.
\subsection{Improved Bounds for Binary Classification}

\label{sec:improved}

In this section, we focus on the important special case where the labels are binary and the loss function is the classification loss $\indicator{\hy\neq y}$.
We present \cref{alg:FTPL} that achieves regret $\widetilde{O}( \sqrt{ {Td}{\sigma^{-{1}/{2}}} })$ with better dependence on the smoothness parameter $\sigma$ compared to \cref{alg:real-valued}.

 This algorithm which is based on the the Follow-the-Perturbed-Leader (FTPL) framework differs from the hint-based algorithm in two ways.
First, instead of accessing the ERM oracle twice and making a randomized decision at each round, the new algorithm only calls the ERM oracle once and follows the prediction of the output. Second, and importantly, the number of self-generated hints does not shrink over time and follows an appropriate probability distribution (such as Poisson). This \emph{Poissonization} is the key to establishing improved stability of the algorithm.

A detailed description of the algorithm is illustrated below. At a high level, at each time step the learner generates $N\sim \text{Poi}(n)$ to represent the total number of self-generated hints it will be feeding itself, and then takes $N$ samples $\widetilde{x}_i \sim \unif(\cX)$, each labeled independently by a Rademacher variable $\widetilde{y}_i$.
The learner then uses the ERM oracle on the history of the play so far and the hints to compute a hypothesis $h_t$, which it uses for its prediction $\widehat{y}_t = h_t(x_t)$ at time $t$.

\begin{algorithm}
\caption{Smoothed Online Binary Classification based on Poisson Number of Hints}
\label{alg:FTPL}
\DontPrintSemicolon
\LinesNumbered
\KwIn{time horizon $T$, smoothness parameter $\sigma$, VC dimension $d$}
$n \gets \min\{ T / \sqrt{\sigma}, T\sqrt{|\cX|/d}\}$; \\
\For{$t\leftarrow 1$ \KwTo $T$}{
    generate $N^{(t)}\sim \text{Poi}(n)$ fresh hallucinated samples $(\widetilde{x}_1^{(t)},\widetilde{y}_1^{(t)}), \cdots, (\widetilde{x}_N^{(t)}, \widetilde{y}_N^{(t)})$, which are i.i.d. conditioned on $N$ with $\widetilde{x}_i^{(t)} \sim \unif(\cX)$ and $\widetilde{y}_i^{(t)} \sim \unif(\{\pm 1\})$; \\
    call the ERM oracle to compute $h_t \gets \opt_{\cH,l}\left(\{(\widetilde{x}_i^{(t)}, \widetilde{y}_i^{(t)})\}_{i\in [N^{(t)}]} \cup \{x_\tau, y_\tau\}_{\tau\in [t-1]} \right)$; \\
    observe $x_t$, predict $\widehat{y}_t = h_t(x_t)$, and receive $y_t$. 
}
\end{algorithm}

In the remainder of this section, we present a proof of the regret upper bound $\widetilde{O}(\sqrt{dT\sigma^{-1/2}})$ in \cref{thm:FTPL} when $\sigma \ge d/|\cX|$. The proof of the other case $\sigma < d/|\cX|$ is slightly different and will be presented in \Cref{subsec:FTPL_DS}. 
Similar to \cref{sec:warmup}, we will provide a relaxation-based upper bound for the regret of \cref{alg:FTPL}. Writing $s = (x, y)$ and $L(h, s) = l(h(x),y) = -yh(x)/2 $, the relaxation is defined as
\begin{align}
    \rel_T(\cH \mid s_{1:t}) = \Ex_{R^{(t+1)}} \left[\sup_{h\in \cH} \left(- \sum_{i=1}^{N^{(t+1)}} L(h, \widetilde{s}_i^{\,(t+1)}) - \sum_{\tau = 1}^t L(h, s_\tau)\right) \right] + \eta(T-t), 
    \label{eq:relaxation_FTPL}
\end{align}
where 
\begin{align}
\eta = \frac{1}{\sqrt{n\sigma}} + c\sqrt{\frac{d\log T}{n\sigma}} + \frac{n\sigma}{4T^2\log T} + e^{-n/8} \in \widetilde{O}\left(\sqrt{\frac{d}{n\sigma}}\right), 
\end{align}
with an absolute constant $c>0$ given in \Cref{lemma:gen_error} later, and $R^{(t)} = (N^{(t)}, \{\widetilde{s}_i\}_{i\in N^{(t)}})$ is the fresh randomness generated at the beginning of time $t$, which is independent of $\{s_\tau\}_{\tau<t}$ generated by the adversary. The relaxation here is similar to \cref{eq:relaxation-transductive} in the transductive setting, where the key difference is a different generation process for the hint set and an additional term $\eta(T-t)$ to account for the stability.

Let $\cQ_t$ be the distribution of the learner's action $h_t\in \cH$ in \cref{alg:FTPL}, then the relaxation in \cref{eq:relaxation_FTPL} is admissible with respect to \cref{alg:FTPL} if the following two conditions hold: 
\begin{align}
\sup_{\cD_t \in \Delta_\sigma(\cX)}\Ex_{x_t\sim \cD_t} \sup_{y_t} \left[\Ex_{h_t\sim \cQ_t}[L(h_t, s_t)] + \rel_T(\cH\mid s_{1:t}) \right] &\le \rel_T(\cH\mid s_{1:t-1}), \quad \forall s_{1:t-1}  \label[Ineq]{eq:relaxation_FTPL_cond_1} \\
\rel_T(\cH\mid s_{1:T}) &\ge -\inf_{h\in \cH}  L(h,s_{1:T}). \label[Ineq]{eq:relaxation_FTPL_cond_2}
\end{align}
According to \Cref{thm:admissibility}, if both \cref{eq:relaxation_FTPL_cond_1,eq:relaxation_FTPL_cond_2} hold, the expected regret of \cref{alg:FTPL} will satisfy
\begin{align}
\Ex[\regret(T)] &\le \rel_T(\cH\mid \emptyset) + O(\sqrt{T}) = \Ex_{R^{(1)}} \left[\sup_{h\in \cH} \left(- \sum_{i=1}^{N^{(1)}} L(h, \widetilde{s}_i^{\,(1)})\right) \right] + \eta T + O(\sqrt{T})\\
&\stepa{=} O\left( \Ex_{N^{(1)}}\left[\sqrt{dN^{(1)}}\right] + \eta T + \sqrt{T} \right) \stepb{=} O\left(\sqrt{dn} + \eta T + \sqrt{T} \right)  , 
\end{align}
and \cref{thm:FTPL} follows from the choices $n = T/\sqrt{\sigma}$ and $\eta = O(\sqrt{d/n\sigma})$. In the above inequality, step (a) follows from random labels and the upper bound $O(\sqrt{nd})$ on the Rademacher complexity of $\cH$ over $n$ points, and step (b) is due to Jensen's inequality and $\Ex[N^{(1)}]=n$. 

Now it remains to verify \cref{eq:relaxation_FTPL_cond_1,eq:relaxation_FTPL_cond_2}. It is not hard to  verify \cref{eq:relaxation_FTPL_cond_2}: this follows from the fact that for any random variable $\lambda$, $\Ex[\sup_{\lambda} X_\lambda] \ge \sup_{\lambda} \Ex[X_\lambda]$ and $\Ex_R[L(h,\widetilde{s}_i)]=0$. The key technical difficulty is in the proof of \cref{eq:relaxation_FTPL_cond_1}. To overcome this challenge,
we first draw a parallel between two types of analysis in online learning, by showing that the \emph{stability} of learner's distribution $\cQ_t$ implies the admissibility of the relaxation, where the stability is measured via $$\text{Stability} = \Ex_{s_t\sim \cD_t}(\Ex_{h_t\sim \cQ_t}[L(h_t,s_t)] - \Ex_{h_{t+1}\sim \cQ_{t+1}}[L(h_{t+1},s_t)]).$$ Note that here $s_t\sim \cD_t$ denotes both the instance and its label and $\cD_t$'s marginal over $\cX$ is $\sigma$-smooth. 
We further upper bound stability (see Lemma~\ref{lemma:stability_admiss}) as follows: 
\[ \text{Stability} \leq 
\text{\rm TV}(\cQ_t, \Ex_{s_t\sim \cD_t}[\cQ_{t+1}]) + \Ex_{s_t,s_t'\sim \cD_t; R^{(t+1)}}[L(h_{t+1}, s_t') - L(h_{t+1},s_t)].
\]
Let us first describe the two terms in this bound further.
The first term is the total variation (TV) distance between $\cQ_t$ and the mixture distribution $\Ex_{s_t\sim \cD_t}[\cQ_{t+1}]$. Note that this TV distance would be an upper bound on the stability by itself, if neither of $\cQ_t$ and $\cQ_{t+1}$ depend on the new observation $s_t=(x_t, y_t)$ at time $t$. However, while $\cQ_t$ is independent of $s_t$, 
$h_{t+1}$ is trained on $s_t$ and thus distribution $\cQ_{t+1}$ does depend on $s_t$. To overcome this dependence, we introduce a ghost sample $s_t'$
that allows us to decouple $h_{t+1}\sim \cQ_{t+1}$ and the new observation. This gives rise to the second term which is a \emph{modified generalization error}. In other words, let $s_t'$ be an independent copy of $s_t$ conditioned on $s_{1:t-1}$. The expected loss of the classifier $h_{t+1}$, which is trained on $s_t$ but not $s_t'$, on the ghost sample $s_t'$ is denoted by
$
\Ex_{s_t,s_t'\sim \cD_t; R^{(t+1)}}[L(h_{t+1}, s_t')]. 
$
The expectation here is understood in the sense that ERM classifier $h_{t+1}$ is determined by the self-generated samples, the current observation $s_t$, and the history $s_1,\cdots,s_{t-1}$ (which are held fixed). The generalization error is then defined to be the expected difference $$\text{Modified generalization error} := \Ex_{s_t,s_t'\sim \cD_t; R^{(t+1)}}[L(h_{t+1},s_t') - L(h_{t+1},s_t)].$$

The following lemma formalizes this discussion and shows that a small TV distance and generalization error suffice to ensure the stability of the algorithm, which in turn implies the admissibility of the relaxation. This result could be of independent interest. 
The proof can be found in \cref{sec:proof_stability_admiss}. 

\begin{lemma}[TV + Generalization $\Rightarrow$ Stability $\Rightarrow$ Admissibility]\label{lemma:stability_admiss}
Let $\cQ_t$ denote learner's distribution over $\cH$ in \cref{alg:FTPL} at round $t$, $\cD_t$ be adversary's distribution at time $t$ (given the history $s_1,\cdots,s_{t-1}$),  $s_t\sim \cD_{t}$ be the realized adversarial instance at time $t$, and $s_t'$ be an independent copy $s'_t\sim \cD_t$. It holds that
\begin{align}
    &\Ex_{s_t\sim \cD_t}\left( \Ex_{h_t\sim \cQ_t}[L(h_t, s_t)] + \emph{\rel}_T(\cH\mid s_{1:t}) \right) - \emph{\rel}_T(\cH \mid s_{1:t-1}) \\
    &\le \Ex_{s_t\sim \cD_t}\left( \Ex_{h_t\sim \cQ_t}[L(h_t,s_t)] - \Ex_{h_{t+1}\sim \cQ_{t+1}}[L(h_{t+1},s_t)]\right) - \eta \\
    &\le \text{\rm TV}(\cQ_t, \Ex_{s_t\sim \cD_t}[\cQ_{t+1}]) + \Ex_{s_t,s_t'\sim \cD_t; R^{(t+1)}}[L(h_{t+1}, s_t') - L(h_{t+1},s_t)] - \eta. 
\end{align}
\end{lemma}

\Cref{lemma:stability_admiss} shows that, in order to prove the admissibility of the relaxation in \Cref{eq:relaxation_FTPL}, it remains to upper bound the TV distance and the generalization error, respectively. 

To upper bound the TV distance, we shall convert $\cQ_t$ to a simpler distribution to work with. For $t\in [T]\cup \{0\}$, let $r^t \in \mathbb{Z}^\cX$ be the $|\cX|$-dimensional random vector with $r^t(x)$ defined to be the difference between the number of $+1$ and $-1$ labels in the self-generated samples and the history up to time $t$ on instance $x$. Formally, 
\begin{align}
    r^t(x) = \sum_{i=1}^{N^{(t+1)}} \widetilde{y}_i^{\,(t+1)}\cdot \mathbf{1}(\widetilde{x}_i^{\,(t+1)} = x) + \sum_{\tau=1}^t y_\tau \cdot \mathbf{1}(x_\tau = x). 
\end{align}
Let $\mathcal{P}^t$ be the distribution of $r^t$. The reason why we introduce this notion is that $h_t$ in \cref{alg:FTPL} only depends on the vector $r^{t-1}$, so the ERM objective
could be written as a quantity depending only on $r^{t-1}$ and $h\in \cH$. We write $h_t = \opt_{\cH, l}(r^{t-1})$ in the sequel, and then $\opt_{\cH, l}(r^{t-1})\sim \cQ_t$ as $r^{t-1}\sim \mathcal{P}^{t-1}$. Therefore, the data-processing inequality shows that 
\begin{align}
    \text{TV}(\cQ_t,\Ex_{s_t\sim \cD_t}[\cQ_{t+1}])\le \text{TV}(\mathcal{P}^{t-1},\Ex_{s_t\sim \cD_t}[\mathcal{P}^t]), 
\end{align}
and the following lemma provides an upper bound on the TV distance $\text{TV}(\mathcal{P}^{t-1},\Ex_{s_t\sim \cD_t}[\mathcal{P}^t])$. 

\begin{lemma}[Upper Bound of TV Distance]\label{lemma:stability}
Let $\mathcal{P}^t$ be the distribution over $r^t$ defined above. We have
\begin{align}
    \sup_{\cD_t \in \Delta_\sigma(\mathcal{S})}\text{\rm TV}(\mathcal{P}^{t-1}, \Ex_{s_t\sim \cD_t}[\mathcal{P}^t]) \le \frac{1}{\sqrt{n\sigma}}. 
\end{align}
\end{lemma}

The key ingredient in the proof of \Cref{lemma:stability} is the Poissonization, which ensures the independence across the coordinates of $r^t$ and enables us to write down the mixture distribution $\Ex_{s_t\sim \cD_t}[\mathcal{P}^t]$ in a compact form. The proof of \Cref{lemma:stability} is presented in \Cref{subsec:TV_proof}. 

The following lemma upper bounds the generalization error for any smooth distribution $\cD_t$. 
\begin{lemma}[Upper Bound of Generalization Error]\label{lemma:gen_error}
Under the notations of \Cref{lemma:stability_admiss}, it holds for an absolute constant $c>0$ (independent of $(n,d,T,\sigma)$) that
\begin{align}
   \sup_{\cD_t\in \Delta_\sigma(\mathcal{X})} \left\{ \Ex_{s_t,s_t'\sim \cD_t; R^{(t+1)}}\left[ L(h_{t+1}, s_t') - L(h_{t+1},s_t) \right] \right\} \le c\sqrt{\frac{d\log T}{n\sigma}} + \frac{n\sigma}{4T^2\log T} + e^{-n/8}. 
\end{align}
\end{lemma}
The intuitive idea behind \Cref{lemma:gen_error} is as follows. Consider the simpler setting of $t = 1$ (i.e. no history) and $\cD_t = \unif(\cX\times \{\pm 1\})$ (i.e. the new observation $s_t$ follows the same distribution as the self-generated samples). In this case, the generalization error is precisely the difference between the test error and the training error with $N+1$ iid training data, and classical Rademacher complexity gives an upper bound $O(\sqrt{d/n})$. For general $\sigma$-smooth $\cD_t$, a coupling argument essentially shows that $n$ iid training data from $\unif(\cX\times \{\pm 1\})$ contain $n\sigma$ iid training data from $\cD_t$, and replacing $n$ by $n\sigma$ in the previous upper bound gives \Cref{lemma:gen_error}. The rigorous treatment of the coupling and all details are presented in \Cref{subsec:gen_error_proof}. 

Now the claimed result of \Cref{thm:FTPL} when $\sigma\ge d/|\cX|$ follows from \Cref{lemma:stability_admiss}, \Cref{lemma:stability}, and \Cref{lemma:gen_error}.

\subsubsection{Upper Bounding TV Distance: Proof of \Cref{lemma:stability}}\label{subsec:TV_proof}
Let us first create a better understanding of the structures of the distributions $\mathcal{P}^{t-1}$ and $\mathcal{P}^t$. Without loss of generality we assume that $\cX$ is discrete (the case of continuous $ \cX $ can be dealt by analyzing the appropriate Poisson point process). Let $n_+(x), n_-(x)$ be the numbers of $+1$ and $-1$ labels, respectively, given instance $x$ in the self-generated samples: 
$$    n_+(x) = \sum_{i=1}^N \mathbf{1}(\widetilde{x}_i = x, \widetilde{y}_i = +1) ~~\text{ and }~~ n_-(x) = \sum_{i=1}^N \mathbf{1}(\widetilde{x}_i = x, \widetilde{y}_i = -1). $$
As each $\widetilde{x}_i$ is uniformly distributed on $\cX$ and $\widetilde{y}_i\sim \unif(\{\pm 1\})$, by the subsampling property of the Poisson distribution, the $2|\cX|$ random variables $\{n_{\pm}(x)\}_{x\in \cX}$ are i.i.d. distributed as $\text{Poi}(n/2|\cX|)$. 
This independence implied by the Poisson distribution plays a key role in the analysis.
Moreover, $r^0(x) = n_+(x) - n_-(x)$, so $\mathcal{P}^0$ is determined by the joint distribution of $\{n_{\pm}(x)\}_{x\in \cX}$.

As we move to general $t$, note that the only contribution of the historic data $\{s_\tau\}_{\tau<t}$ to both $\mathcal{P}^{t-1}$ and $\mathcal{P}^t$ is a common translation independent of $\mathcal{P}^0$. Since the TV distance is translation invariant, it suffices to upper bound $\text{TV}(\mathcal{P}^0, \Ex_{s_1}[\mathcal{P}^1])$. Let $n_{\pm}^1(x) = n_{\pm}(x) + \mathbf{1}(x_1=x,y_1=\pm 1)$, it holds that $r^1(x) = n_+^1(x) - n_-^1(x)$. Consequently, let $P$ and $Q$ be the probability distributions of  $\{n_{\pm}(x)\}_{x\in \cX}$ and $\{n_{\pm}^1(x)\}_{x\in \cX}$, respectively, the data-processing inequality implies that $\text{TV}(\mathcal{P}^0, \Ex_{s_1}[\mathcal{P}^1]) \le \text{TV}(P,Q)$.

As discussed above, the distribution $P$ is a product Poisson distribution:
\begin{align}
P(\{n_{\pm}(x)\}) = \prod_{x\in \cX}\prod_{y\in \{\pm\}} \mathbb{P}(\text{Poi}(n/2|\cX|) = n_y(x)). 
\end{align}
As for the distribution $Q$, it could be obtained from $P$ in the following way: the smooth adversary draws $x^\star \sim \cD $, independent of $\{n_{\pm}(x)\}_{x\in \cX}\sim P$, for some $\sigma$-smooth distribution $\cD\in \Delta_\sigma(\cX)$. He then chooses a label $y^\star = y(x^\star)\in \{\pm 1\}$ as a function of $x^\star$, and sets 
\begin{align}
    n_{y(x^\star)}^1(x^\star) = n_{y(x^\star)}(x^\star) + 1, \qquad \text{and} \qquad  n_y^1(x) = n_y(x), \quad \forall (x,y) \neq (x^\star, y(x^\star)). 
\end{align}
Consequently, given a $\sigma$-smooth distribution $\cD$ and a labeling function $y: \cX\to \{\pm \}$ used by the adversary, the distribution $Q$ is a mixture distribution $Q = \Ex_{x^\star \sim \cD^\cX}[Q_{x^\star}]$, with
\begin{align}
    Q_{x^\star}(\{n_{\pm}(x)\}) = \mathbb{P}(\text{Poi}(n/2|\cX|) = n_{y(x^\star)}(x^\star)-1)\times \prod_{(x,y)\neq (x^\star, y(x^\star))} \mathbb{P}(\text{Poi}(n/2|\cX|) = n_y(x)). 
\end{align}

To upper bound the TV distance between a mixture distribution $Q$ and a base distribution $P$, we will rely on the smoothness properties of $\cD$, in particular, that the probability of collision between two independent draws $x_1^\star, x_2^\star \sim \cD$ is small. To formally address this, we make use of two technical lemmas, first to upperbound the TV distance in terms of the $\chi^2$ distance, and second to use the Ingster's method for bounding the $\chi^2$ distance between a mixture distribution and a base distribution. See \Cref{lemma:TV_chi2} and \Cref{lemma:chi2_mixture} in the \Cref{app:auxlemma-chi2} for more details.
Let  $x_1^\star, x_2^\star$ be an arbitrary pair of instance. 
Using the closed-form expressions of distributions $P$ and $Q_{x^\star}$, it holds that
\begin{align}
\frac{Q_{x_1^\star}(\{n_{\pm}(x)\}) Q_{x_2^\star}(\{n_{\pm}(x)\})}{P(\{n_{\pm}(x)\})^2} = \frac{2|\cX|n_{y(x_1^\star)}(x_1^\star)}{n}\cdot \frac{2|\cX|n_{y(x_2^\star)}(x_2^\star)}{n}. 
\end{align}
Using the fact that $\{n_{\pm}(x)\}_{x\in \cX}$ are i.i.d. distributed as $\text{Poi}(n/2|\cX|)$ under $P$, we have
\begin{align}
    \Ex_{\{n_{\pm}(x)\}\sim P}\left(\frac{Q_{x_1^\star}(\{n_{\pm}(x)\}) Q_{x_2^\star}(\{n_{\pm}(x)\})}{P(\{n_{\pm}(x)\})^2}\right) = 1 + \frac{2|\cX|}{n}\cdot \mathbf{1}(x_1^\star = x_2^\star). 
\end{align}
Now using the aforementioned lemmas (\Cref{lemma:TV_chi2} and \Cref{lemma:chi2_mixture}), we have
\begin{align}
\text{TV}(P,Q) & \leq \sqrt{\frac{\chi^2(Q, P)}{2}} = \sqrt{\frac{\chi^2(\Ex_{x^\star\sim \cD}[Q_{x^\star}],P)}{2}} = \sqrt{\frac{|\cX|}{n}\cdot \Ex_{x_1^\star, x_2^\star\sim \cD}[\mathbf{1}(x_1^\star = x_2^\star)]} \\
&= \sqrt{\frac{|\cX|}{n}\sum_{x\in\cX} \cD(x)^2} \stepa{\le} \sqrt{\frac{|\cX|}{n}\sum_{x\in\cX} \cD(x)\cdot \frac{1}{\sigma|\cX|}} = \frac{1}{\sqrt{\sigma n}},
\end{align}
where (a) follows from the definition of a $\sigma$-smooth distribution. This completes the proof.
 
\subsubsection{Upper Bounding Generalization Error: Proof of \Cref{lemma:gen_error}}\label{subsec:gen_error_proof}
In the proof of \Cref{lemma:gen_error}, we shall need the following property of smooth distributions which is a slightly strengthened version of the coupling lemma in \Cref{lemma:coupling}. 

\begin{lemma}\label{lemma:coupling_strong}
Let $X_1,\cdots,X_m\sim Q$ and $P$ be another distribution with a bounded likelihood ratio: $dP/dQ \le 1/\sigma$. Then using external randomness $R$, there exists an index $I = I(X_1,\cdots,X_m,R) \in [m]$ and a success event $E = E(X_1,\cdots,X_m,R)$ such that $\Pr[E^c] \le (1-\sigma)^m$, and
\begin{align}
    (X_I \mid E, X_{\backslash I}) \sim P. 
\end{align}
\end{lemma}

Fix any realization of the Poissonized sample size $N\sim \text{Poi}(n)$. Choose $m = 4\sigma^{-1}\log T$ in \Cref{lemma:coupling_strong}, and without loss of generality assume that $N$ is an integral multiple of $m$. Since for any $\sigma$-smooth  $\cD_t$, it holds that
\begin{align}
    \frac{\cD_t(s)}{\unif(\cX\times \{\pm 1\})(s)} = \frac{\cD_t(x)}{\unif(\cX)(x)}\cdot \frac{\cD_t(y\mid x)}{\unif(\{\pm 1\})(y)} \le \frac{2}{\sigma}, 
\end{align}
the premise of \Cref{lemma:coupling_strong} holds with parameter $\sigma/2$ for $P = \cD_t, Q = \unif(\cX\times \{\pm 1\})$. Consequently, dividing the self-generated samples $\widetilde{s}_1, \cdots, \widetilde{s}_N$ into $N/m$ groups each of size $m$, and running the procedure in \Cref{lemma:coupling_strong}, we arrive at $N/m$ independent events $E_1, \cdots, E_{N/m}$, each with probability at least $1-(1-\sigma/2)^m\ge 1-T^{-2}$. Moreover, conditioned on each $E_j$, we can pick an element $u_j \in \{\widetilde{s}_{(j-1)m+1}, \cdots, \widetilde{s}_{jm} \}$ such that
\begin{align}
    (u_j \mid E_j, \{\widetilde{s}_{(j-1)m+1}, \cdots, \widetilde{s}_{jm} \} \backslash \{u_j\}) \sim \cD_t. 
\end{align}
For notational simplicity we denote the set of unpicked samples $\{\widetilde{s}_{(j-1)m+1}, \cdots, \widetilde{s}_{jm} \} \backslash \{u_j\}$ by $v_j$. As a result, thanks to the mutual independence of different groups and $s_t\sim \cD_t$ conditioned on $s_{1:t-1}$ (note that we draw fresh randomness at every round), for $E\triangleq \cap_{j\in [N/m]} E_j$ we have
\begin{align}
    (u_1, \cdots, u_{N/m}, s_t) \mid (E, s_{1:t-1}, v_1, \cdots, v_{N/m}) \overset{\text{iid}}{\sim} \cD_t. 
\end{align}

Consequently, for each $j\in [N/m]$ we have
\begin{align}
&\Ex_{s_t\sim \cD_t, R^{(t+1)}}[L(h_{t+1},s_t) \mid E] \\
&= \Ex_{s_t\sim \cD_t, \widetilde{s}_1, \cdots, \widetilde{s}_N} \left[ L(\opt(\widetilde{s}_1,\cdots,\widetilde{s}_N,s_{1:t-1},s_t), s_t)\mid E  \right] \\
&= \Ex_{v,s_{1:t-1}\mid E}\left( \Ex_{s_t,u_1,\cdots,u_{N/m}}\left[ L(\opt(s_{1:t-1},v,u_1,\cdots,u_{N/m},s_t), s_t) \mid E,s_{1:t-1},v \right] \right) \\
&\stepa{=} \Ex_{v,s_{1:t-1}\mid E}\left( \Ex_{s_t,u_1,\cdots,u_{N/m}}\left[ L(\opt(s_{1:t-1},v,u_1,\cdots,u_{j-1},s_t,u_{j+1},\cdots,u_{N/m},u_j), u_j) \mid E,s_{1:t-1},v \right] \right) \\
&\stepb{=} \Ex_{v,s_{1:t-1}\mid E}\left( \Ex_{s_t,u_1,\cdots,u_{N/m}}[ L(\opt(s_{1:t-1},v,u_1,\cdots,u_{N/m},s_t), u_j) \mid E,s_{1:t-1},v ] \right)  \\
&= \Ex_{s_t\sim \cD_t, R^{(t+1)}}[L(h_{t+1},u_j)\mid E], 
\end{align}

where (a) follows from the conditional iid (and thus exchangeable) property of $(u_1,\cdots,u_{N/m},s_t)$ after the conditioning, and (b) is due to the invariance of the ERM output after any permutation of the inputs. On the other hand, if $s_t', u_1', \cdots, u_{N/m}'$ are independent copies of $s_t\sim \cD_t$, by independence it is clear that
\begin{align}
\Ex_{s_t,s_t'\sim \cD_t, R^{(t+1)}}[L(h_{t+1},s_t') \mid E] = \Ex_{s_t,s_t'\sim \cD_t, R^{(t+1)}}[L(h_{t+1},u_j') \mid E], \quad \forall j \in [N/m]. 
\end{align}
Consequently, using the shorthand $u_0 = s_t, u_0' = s_t'$, we have
\begin{align}
&\Ex_{s_t,s_t'\sim \cD_t, R^{(t+1)}}[L(h_{t+1},s_t') - L(h_{t+1},s_t) \mid E] \\
&= \frac{1}{N/m + 1}\Ex_{s_t,s_t'\sim \cD_t, R^{(t+1)}}\left[\sum_{j=0}^{N/m}(L(h_{t+1},u_j') - L(h_{t+1},u_j)) ~ \bigg| ~ E\right] \\
&\le \frac{1}{N/m + 1}\Ex_{u_0,\cdots,u_{N/m},u_0',\cdots,u_{N/m}'\sim \cD_t}\left[\sup_{h\in \cH}\sum_{j=0}^{N/m} (L(h,u_j') - L(h,u_j))\right] \\
&\le \frac{2}{N/m+1} \Ex_{u_0,\cdots,u_{N/m}\sim \cD_t} \Ex_{ \epsilon_1 \dots \epsilon_{N/m}} \left[\sup_{h\in \cH}\sum_{j=0}^{N/m}\epsilon_j h(u_j) \right] \le c_0\sqrt{\frac{d}{N/m+1}},
\end{align}
where the last inequality is due to the classical $O(\sqrt{d/n})$ upper bound on the Rademacher complexity, and $c_0>0$ in an absolute constant. Note that the union bound gives
\begin{align}
    \Pr[E^c] \le \sum_{j=1}^{N/m} \Pr[E_j^c] \le \frac{N}{mT^2},
\end{align}
the law of total expectation gives
\begin{align}
&\Ex_{s_t,s_t'\sim \cD_t, R^{(t+1)}}[L(h_{t+1},s_t') - L(h_{t+1},s_t)] \\
&\le \Ex_{s_t,s_t'\sim \cD_t, R^{(t+1)}}[L(h_{t+1},s_t') - L(h_{t+1},s_t) \mid E] + \Pr[E^c] \le c_0\sqrt{\frac{d}{N/m+1}} + \frac{N}{mT^2}. 
\end{align}

Finally, plugging the choice of $m = 4\sigma^{-1}\log T$, taking the expectation of $N\sim \text{Poi}(n)$, and using $\Pr[N>n/2]\ge 1-e^{-n/8}$ in the above inequality completes the proof of \Cref{lemma:gen_error}.

\subsubsection{Completing the Proof of \Cref{thm:FTPL}}\label{subsec:FTPL_DS}
In this section we complete the proof of the $O(\sqrt{T(d|\cX|)^{1/2}})$ upper bound in \Cref{thm:FTPL} when $\sigma < d/|\cX|$ (and thus $n=T\sqrt{|\cX|/d}$). The proof is still through the same relaxation in \Cref{eq:relaxation_FTPL}, though we will choose a different parameter $\eta$ and prove a slightly modified version of \Cref{lemma:stability_admiss}: 
\begin{lemma}[Expected TV $\Rightarrow$ Admissibility]\label{lemma:stability_admiss_modified}
Let $\cQ_t$ denote learner's distribution over $\cH$ in \cref{alg:FTPL} at round $t$, and $s_t\sim \cD_t$ be the conditional distribution of $s_t$ given the history $s_1,\cdots,s_{t-1}$. It holds that
\begin{align}
    \Ex_{s_t\sim \cD_t}\left( \Ex_{h_t\sim \cQ_t}[L(h_t, s_t)] + \emph{\rel}_T(\cH\mid s_{1:t}) \right) - \emph{\rel}_T(\cH \mid s_{1:t-1}) \le \Ex_{s_t\sim \cD_t}[\text{\rm TV}(\cQ_t, \cQ_{t+1})] - \eta. 
\end{align}
\end{lemma}

Note that in \Cref{lemma:stability_admiss_modified}, the expectation is outside the TV distance and no smaller than the TV distance when the mixture distribution is inside the expectation compared with \Cref{lemma:stability_admiss}. We can simply upper bound this expected TV distance, with the worst case choice of $s_t$ and apply the data processing inequality, i.e., 
\begin{align}
\Ex_{s_t\sim \cD_t}[\text{\rm TV}(\cQ_t, \cQ_{t+1})] \le \sup_{s_t} \text{TV}(\mathcal{P}^{t-1}, \mathcal{P}^t). 
\end{align}
Using the similar independence property of Poissonization in \Cref{subsec:TV_proof}, the target TV distance is at most $\text{TV}(P,Q)$, where $P\sim \text{Poi}(n/2|\cX|)$, and $Q$ is a right-translation of $P$ by one. Consequently, 
\begin{align}
    \text{TV}(P,Q) \le \sqrt{\frac{\chi^2(Q,P)}{2}} = \sqrt{\frac{1}{2}\left(\Ex_{X\sim P}\left[\left(\frac{X}{n/2|\cX|}\right)^2 \right] - 1\right)} = \sqrt{\frac{|\cX|}{n}}, 
\end{align}
so the choice of $\eta = \sqrt{|\cX|/n}$ and \Cref{lemma:stability_admiss_modified} again makes the relaxation in \Cref{eq:relaxation_FTPL} admissible, and we complete the proof of \Cref{thm:FTPL}. 
 
\section{Discussion, Additional Results, and Open Problems}
\label{sec:discussion}
\paragraph{\textbf{Computational Lower Bounds.}} 
Our main contribution is oracle-efficient algorithms that achieve an $O(\sqrt{dTK})$ regret upper bound for $K$-hint transductive learning, an $O(\sqrt{dT \sigma^{-1} })$ regret upper bound in the real-valued case of smoothed online learning, and an $O(\sqrt{dT\sigma^{-1/2}})$ upper bound for smoothed online binary classification. However, neither of these upper bounds is statistically optimal; the statistically optimal regrets for these scenarios are $\widetilde{\Theta}(\sqrt{dT\log K})$ and $\widetilde{\Theta}(\sqrt{dT\log(1/\sigma)})$ \cite{haghtalab2022smoothed}, respectively. We ask the following question: is the above discrepancy an artifact of our regret analysis, or an intrinsic limitation of our or all oracle-efficient  algorithms.

We show that no matter how we tune the parameters in \cref{alg:FTPL} and \cref{alg:transductive_k_hints}, a lower bound $\Omega(\sqrt{dT\sigma^{-1/2}})$ or $\Omega(\sqrt{dTK^{1/2}})$ is unavoidable for these algorithms in the respective scenarios. In other words, our regret upper bound in \cref{thm:FTPL} is tight, and there is an $O(K^{1/4})$ gap compared to the upper bound of \cref{thm:regret-transductive}; importantly, neither algorithm could achieve the logarithmic dependence on $1/\sigma$ or $K$ in the statistical upper bound. Formally, by parameter tuning, we mean any choice of the parameter $n$ in \cref{alg:FTPL}, and any number $n$ of uniform samples from the $K$-hint set in \cref{alg:transductive_k_hints}. The next theorem formally states the lower bounds. 

\begin{theorem}[Limitations of Algorithms]\label{thm:lower_bound}
For any choice of the parameter $n$ in \cref{alg:FTPL}, there exists a $\sigma$-smoothed online learning instance such that \cref{alg:FTPL} suffers from at least $\Omega(\min\{T, \sqrt{dT\sigma^{-1/2}}, \sqrt{T(d|\cX|)^{1/2}}\})$ expected regret. 

Similarly, for any choice of the parameter $n$ in \cref{alg:transductive_k_hints}, there exists a $K$-hint transductive instance such that \cref{alg:transductive_k_hints} incurs at least $\Omega(\min\{T, \sqrt{dTK^{1/2}}\})$ expected regret. 
\end{theorem} 

For general efficient algorithms, we have the following computational lower bound for smoothed online learning following similar ideas to \cite{HazanKoren}. 

\begin{theorem}[Computational Lower Bound for Smoothed Online Learning]\label{thm:comp_lower_bound}
For $1/\sigma \ge d$, any proper algorithm which only has access to the ERM oracle and achieves a regret $o(\min\{T, \sqrt{T(d/\sigma)^{1/2}}\})$ for any $\sigma$-smoothed online learning problem must have an $\omega(\sqrt{d/\sigma})$ total running time.
\end{theorem}

\cref{thm:comp_lower_bound} implies an exponential statistical-computational gap in smoothed online learning: for exponentially small $\sigma$, achieving the statistical regret $\widetilde{O}(\sqrt{Td\log(1/\sigma)})$ requires an exponential running time. However, \cref{thm:comp_lower_bound} still exhibits gaps to our computational upper bounds. 
\begin{enumerate}
    \item First, although both lower bounds of the regret and running time in \cref{thm:comp_lower_bound} does not match the counterparts of \cref{alg:FTPL} in \Cref{thm:FTPL}, the upper and lower bounds share the same $\Theta(\sigma^{-1/4})$ dependence on $\sigma$. This suggests that the improvement from $\Theta(\sigma^{-1/2})$ to $\Theta(\sigma^{-1/4})$ thanks to Poissonization is not superfluous and might be fundamental. We also conjecture that for all efficient algorithms with runtime $\text{poly}(T,d,1/\sigma)$, the $\Theta(\sigma^{-1/4})$ dependence is the best one can hope for in the regret of such algorithms, as opposed to the $\Theta(\sqrt{\log(1/\sigma)})$ dependence in the statistical regret. 
    
    \item Second, \cref{thm:comp_lower_bound} shows a $\text{poly}(d,1/\sigma)$ computational lower bound to achieve the statistical regret $\widetilde{O}(\sqrt{Td\log(1/\sigma)})$, while the $\varepsilon$-net argument in \cite{haghtalab2022smoothed} requires a $\text{poly}(\sigma^{-d})$ computational time. One may wonder whether this exponential dependence on $d$ is in fact unavoidable, and this is a missing feature not covered in \cite{HazanKoren}. We ask the following open question: 
\paragraph{\textbf{Open Question.}} \emph{For $d/\sigma \gg T^2$ in the smoothed setting (or $dK\gg T$ in the $K$-hint transductive learning setting), does any algorithm achieving $o(T)$ regret require $\Omega(\text{\rm poly}(T,2^d,1/\sigma (\text{ or } K)))$ computational time given access to the ERM oracle?}
\end{enumerate}

\section*{Acknowledgments}
    This work was supported in part by the National Science Foundation under grant CCF-2145898, a C3.AI
    Digital Transformation Institute grant, and Berkeley AI Research Commons grants. This work was partially done while authors were visitors at the Simons Institute for the Theory of Computing. \bibliographystyle{alpha}
\bibliography{Ref_New.bib}

\appendix

\section{Proof of Proposition 2.1}
\label{appendix:admissibility}
\begin{proof}
     To prove this lemma we break the expected regret into two parts:
     \begin{align}
         \Ex[\regret(T)]=\Ex_{\sD,\sQ}\left[
	    \sum_{t=1}^T 
	    \Ex_{\hy_t\sim\cQ_t}[l(\hy_t,y_t)]
	    \!-\!\inf_{h\in\cH}\!L(h,s_{1:T})
	    \right]+\Ex_{\sD,\sQ}\left[
	    \sum_{t=1}^T 
	    l(\hy_t,y_t)
	    \!-\!\Ex_{\hy_t\sim\cQ_t}\![l(\hy_t,y_t)]
	    \right].
     \end{align}
     
     For the first part, we use an inductive argument to show that
     \begin{align}
         \Ex_{\sD,\sQ}\left[
	    \sum_{t=1}^T 
	    \Ex_{\hy_t\sim\cQ_t}[l(\hy_t,y_t)]
	    -\inf_{h\in\cH}\sum_{t=1}^T l(h({x_t}),y_t)
	    \right]\le \rel_T(\cH\mid \emptyset).\label{eq:induction-goal}
     \end{align}
         According to the definition of admissibility, we have\begin{align}
             &\Ex_{\sD,\sQ}\left[\sum_{t=1}^T  \Ex_{\hy_t\sim\cQ_t}[l(\hy_t,y_t)]\underbrace{-\inf_{h\in\cH}
             \sum_{t=1}^T l(h({x_t}),y_t)}_{\le \rel_T(\cH\mid s_{1:T}) \text{ by 2nd condition of admissibility}}\right]\\
             \le&\Ex_{\sD,\sQ}\left[
                \Ex\left[\left.
                    \sum_{t=1}^{T-1} \Ex_{\hy_t\sim\cQ_t}[l(\hy_t,y_t)]
                    +\underbrace{\underset{x_T\sim \cD_T}{\Ex}\Big[
                    \Ex_{\hy_T\sim\cQ_T}[l(\hy_T,y_T)]
                    +\rel_T(\cH\mid s_{1:T})\Big]}_{\le \rel_T(\cH\mid s_{1:T-1})\text{ by 1st condition of admissibility}}
                    \right|s_{1:T-1}
                \right]
             \right]\\
             \le&\Ex_{\sD,\sQ}\left[
                \Ex\left[\left.
                    \sum_{t=1}^{T-1} \Ex_{\hy_t\sim\cQ_t}[l(\hy_t,y_t)]
                    +\rel_T(\cH\mid s_{1:T-1})
                    \right|s_{1:T-1}
                \right]
             \right]\\
             =&\Ex_{\sD,\sQ}\left[
                    \sum_{t=1}^{T-1} \Ex_{\hy_t\sim\cQ_t}[l(\hy_t,y_t)]
                    +\rel_T(\cH\mid s_{1:T-1})
             \right],
         \end{align}
         where the last step uses the tower property of conditional expectations.
         
         Repeat this process for $(T-1)$ times and note that $\rel_T(\cH\mid \emptyset)$ is a constant that does not dependent on $\sD$ proves \cref{eq:induction-goal}.
        
        Since the second part is the expected sum of a martingale difference sequence, we apply the Azuma-Hoeffding inequality and obtain
        \begin{align}
            \Ex_{\sD,\sQ}\left[
	    \sum_{t=1}^T 
	    l(\hy_t,y_t)
	    -\Ex_{\hy_t\sim\cQ_t}[l(\hy_t,y_t)]
	    \right] \le \int_0^\infty \exp\left(-\frac{2t^2}{T}\right)dt \in O(\sqrt{T}).
	    \label{eq:concentration}
        \end{align}
        
        Combining \cref{eq:induction-goal} and \cref{eq:concentration} completes the proof.
    \end{proof}

\section{Transductive Online Learning with $K$ Hints}
\subsection{Efficient Algorithm for Transductive Online Learning with $K$ Hints}
See \cref{alg:transductive_k_hints} for a description of the oracle-efficient algorithm in the setting of transductive online learning with $K$ hints.

\begin{algorithm}
    \caption{Oracle-Efficient Online Transductive Learning with $K$ Hints}
    \label{alg:transductive_k_hints}
    \DontPrintSemicolon
    \LinesNumbered
    \KwIn{$T,K, \{Z_t\}_{t=1}^T$}
    \For{$t\leftarrow 1$ \KwTo $T$}{
        Receive $x_t$. Assert that $x_t\in Z_t.$\;
        \For{$i=t\!+\!1,\cdots,T$; $k=1,\cdots,K$}{
            Draw new $\epsilon_{i,k}^{(t)}\sim\unif(\{-1,+1\})$.
        }
        $S^{(t)}\gets \left\{ (z_{i,k}^{(t)},\epsilon_{i,k}^{(t)})
        \right\}_{\myatop{i=t+1:T}{k=1:K}}$, \;
        \label{alg:random instances}
        $\widehat{y_t}\gets \OPT\left( s_{1:t-1}; S^{(t)}\cup S^{(t)}\cup\{(x_t, -1)\} \right) - 
        \OPT\left( s_{1:t-1}; S^{(t)}\cup S^{(t)}\cup\{(x_t, +1)\} \right)$. \;
        Receive $y_t$, suffer loss $l(\hy_t,y_t)$.
    }
    \end{algorithm}

    \subsection{Monotonicity of Regularized Rademacher Complexity}
\label{appendix:monotonicity}
\medskip
\noindent\textbf{\Cref{lem:monotonicity}} (Restated)\textbf{.}\emph{
    Let $Z=\{z_i\}_{i\in[m]}\in\cX^m$ be a set of unlabeled instances and $\Phi:\cH\to\mathbb{R}$ be a mapping from the set of hypothesis to real values. Recall that the Rademacher complexity for set $Z$ regularized by $\Phi$ is defined as
$$
     \mathfrak{R}({\Phi},Z)= \Ex_{\epsilon_{1:m}\overset{\text{iid}}{\sim}\unif(\pm1)}\left[
        \sup_{h\in\cH}\Big\{\sum_{i=1}^m\epsilon_i h(z_i)
        +\Phi(h)\Big\}
        \right].
        $$
    Then for any dataset $z_{1:m}\in\cX^m$ and any additional data point $x\in \cX$, we have $$\mathfrak{R}(\Phi,z_{1:m})\le\mathfrak{R}(\Phi,z_{1:m}\cup\{x\}).$$
}

\medskip
\begin{proof}[Proof of \Cref{lem:monotonicity}]
         Using $\Ex[\sup_\lambda X_\lambda] \ge \sup_\lambda \Ex[X_\lambda]$, we have
         \begin{align}
             \mathfrak{R}(\Phi,Z\cup\{x\}) &=
             \Ex_{\epsilon_{1:m+1}}\left[
             \sup_{h\in\cH}\left\{\sum_{i=1}^m\epsilon_i h(z_i)
             +\epsilon_{m+1}h(x)+\Phi(h)\right\}\right]\\
             &= \Ex_{\epsilon_{1:m},\epsilon_{m+1}}\left[\sup_{h\in\cH}\left\{\sum_{i=1}^m\epsilon_i h(z_i)
             +\epsilon_{m+1}h(x)+\Phi(h)\right\}
             \right]\\
             &\ge \Ex_{\epsilon_{1:m}}\left[\sup_{h\in\cH}\left\{\sum_{i=1}^m\epsilon_i h(z_i)
             +\Ex_{\epsilon_{m+1}}[\epsilon_{m+1}h(x)]+\Phi(h)\right\}
             \right] \\
             &=\mathfrak{R}({\Phi},Z),
         \end{align}
     as desired.
    \end{proof}

\section{Unknown Smoothness Parameters}
\label{appendix:unknown_sigma}
Suppose we have upper and lower bounds $\sigma_{\max}$ and $\sigma_{\min}$ on the exact value of $\sigma$, i.e., $\sigma_{\min}\le\sigma\le\sigma_{\max}$. In this section, we introduce a meta algorithm that uses a geometric doubling approach to incorporate knowledge of $\sigma_{\max}$ and $\sigma_{\min}$ into the algorithms introduced in \cref{sec:binary}.

We start by constructing $\log (\sigma_{\max}/\sigma_{\min})$ experts, where each expert $i$ runs a local version of our algorithm (can be either \cref{alg:real-valued} for the real-valued case or \cref{alg:FTPL} for the binary case) with parameter $ \sigma_i=2^{i}\cdot\sigma_{\min}$. We then run Hedge on these experts. Note that the parameter $i^\star$ of the best expert satisfies $\frac{\sigma}{2}\le\sigma_{i^\star}\le\sigma$, so the expected regret of this expert matches the expected regret of the same algorithm running on true $\sigma$ up to a constant factor. Therefore, the expected regret of this meta algorithm is comparable to the bound in Theorem 3.1 and 3.2, with an additive term of order at most $O\left(\sqrt{T\log\log(\sigma_{\max}/\sigma_{\min})}\right)$. The number of oracle calls also blows up only by $\log(\sigma_{\max}/\sigma_{\min})$ per round. This could potentially be improved using a more aggressive step size for the Hedge meta algorithm.

\section{Smoothed Online Learning with Real-valued Functions}
\label{appendix:omit-older-proof}
\subsection{Coupling Lemma}
\label{app:coupling}
\begin{lemma}[Coupling, \cite{haghtalab2022smoothed}]
    \label{lemma:coupling}
        Let $ \sD_\sigma $ be an adaptive sequence of $t$ $\sigma$-smooth distributions on $\mathcal{X}$. 
        Then, there is a coupling $\Pi$ such that $ \left( x_1 , z_{1,1} , \dots , z_{1,K}   , \dots, x_t , z_{t,1}, \dots , z_{t,K}  \right) \sim \Pi $ satisfy
        \begin{itemize}
            \item[a.] $x_1 , \dots , x_t$ is distributed according $\sD_\sigma$. 
            \item[b.] For every $j \leq t$, $\{ z_{i,k} \}_{{i \geq j},{k\in[K]}} $ are uniformly and independently distributed on $\mathcal{X}$, conditioned on $x_1 , \dots, x_{j-1}$.  
            \item[c.]\label{item:failure-coupling}  With probability at least $1 - t \left( 1 - \sigma \right)^{K}  $, $ \{x_1,\cdots,x_T\}\subseteq \{z_{t,k}\}_{{t=1:T},{k=1:K}}$ . 
        \end{itemize}
    \end{lemma}

    \subsection{Notions for Real-Valued Functions}

        In this section we introduce the notions that will be useful in analyzing real-valued hypothesis classes, including pseudo dimension and covering numbers.
        
        \newcommand{\sign}{\mathrm{sgn}}
        \begin{definition} [Pseudo-dimension, \cite{anthony1999neural}]
            \label{def:pseudo-dimension}
          For every $h\in\cH$, let $B_h(x,y)=\sign(h(x)-y)$ be the indicator of the region below or on the graph of $h$. The pseudo-dimension of hypothesis class $\cH$ is defined as the VC dimension of the subgraph class $B_\cH=\{B_h:h\in\cH\}$.
        \end{definition}
        
        We will see in the two following lemmas that pseudo dimension can be used to characterize the magnitude of covering numbers and Rademacher complexity.
        
        \begin{lemma}[$d_{L_1(\unif)}$-Covering Number Bound, \cite{anthony1999neural}]
        \label{lemma:covering-number}
            The $\epsilon$-covering number of $\cH$ with respect to metric $d_{L_1(\unif)}$, denoted by $\cN(\epsilon,\cH,L_1(\unif(\cX)))$, is the cardinality of the smallest subset $\cH'$ of $\cH$, such that for every $h\in\cH$, there exists $ h'\in\cH'$ such that $d_{L_1(\unif)}(h,h')\le\epsilon$, where $d_{L_1(\unif)}(f,g)=\Ex_{\unif}[|f-g|]$.
            If $d$ is the pseudo-dimension of $\cH$, then for any $\epsilon>0$, \begin{align}
                \log \cN(\epsilon,\cH,L_1(\unif(\cX)))\in\widetilde{O}\!\left(d\log\!\left(\frac{1}{\epsilon}\right)\right).
            \end{align}
        \end{lemma}

        \begin{lemma}[Rademacher Complexity Bound, \cite{bartlett2008notes}]
        \label{lemma:rademacher-bound}
            The Rademacher complexity of class $\cH$ for a set of $n$ elements is upper bounded by $O\!\left(\sqrt{dn\log n}\right)$, where $d$ is the pseudo dimension of $\cH$.
        \end{lemma}

\subsection{Proof of \cref{thm:regret-real}} 
\label{appendix:proof-real}

\medskip
\noindent\textbf{\cref{thm:regret-real}} (Restated)\textbf{.}\emph{
    For any $\sigma$-smooth adversary $\sD_\sigma$, \Cref{alg:real-valued} has expected regret upper bounded by  $ \widetilde{O}(G\sqrt{Td/\sigma})$, where $\widetilde{O}$ hide factors that are polynomial in $\log(T)$ and $\log(1/\sigma)$.
	Here $ G$ is the Lipschitz constant of the loss and $d$ is the pseudodimension of class $\cH$. 
	Furthermore, the algorithm is oracle-efficient: at every round $t$, this algorithm uses two oracle calls with histories of length $\widetilde{O}(T/\sigma)$.
}

\medskip

\begin{proof}
	To prove \cref{thm:regret-real} we use the following relaxation:
	\begin{align}
		\rel_T(\cH|s_{1:t})=&
        2G\Ex_{V^{(t)}\simiid \unif(\cX)}\left[\mathfrak{R}(-{L^{\mathrm{r}}(\cdot,s_{1:t})},V^{(t)})
       \right]+2G\beta(T-t)\\
        =&2G\Ex_{V^{(t)},\cE^{(t)}}\left[
		\sup_{h\in\cH}\left\{\sum_{i=t+1:T\atop k=1:K}
		\epsilon_{i,k}^{(t)} h(v_{i,k}^{(t)})
		-L^{\mathrm{r}}(h,s_{1:t})
		\right\}
		\right]+2G\beta(T-t),\label{eq:relaxation-real}
	\end{align}
	
	where $K = 100\log T / \sigma $ and $\beta=10TK(1-\sigma)^K$. 
	We will show in \Cref{lemma:real-admissible} that the above relaxation is admissible. Therefore, \Cref{thm:admissibility} gives us the following upper bound on the expected regret:
	\begin{align}
		\Ex[\regret(T)]\le \rel_T(\cH|\emptyset)\!+\!O(\sqrt{T})
		\!=\!2G \underbrace{\Ex_{V^{(0)},\cE^{(0)}}\left[
		\sup_{h\in\cH}\left\{\sum_{i=1:T\atop k=1:K}
		\epsilon_{i,k}^{(0)} h(v_{i,k}^{(0)})
		\right\}
		\right]}_{\text{(a)}
        }+2G\beta T\!+\!O(\sqrt{T}).
	\end{align}
	The first term (a) is the Rademacher complexity of the hypothesis class $\cH$ with respect to the uniform distribution for sample size $TK$.
	By \Cref{lemma:rademacher-bound}, $\text{(a)}\le O\left(\sqrt{dTK\log(TK)}\right)$.
	For the second term, we have $\beta T\in o(1)$ because $\beta\lesssim TK e^{-\sigma K}\lesssim T^{-99}\log T/\sigma =o(T)$. Plugging in $K=O\left({\log(T)}/{\sigma}\right)$, we have the following bound:
	 \begin{align}
	 	\Ex[\regret(T)]\le {O}\left(G\sqrt{\frac{dT}{\sigma}\log T\log\left(
	 		\frac{T}{\sigma}\right)}\right)\subseteq\widetilde{O}(G\sqrt{dT/\sigma}),
	 \end{align}
 	where $\widetilde{O}$ hide factors that are polynomial in $\log(T)$ and $\log(1/\sigma)$.
\end{proof}

\subsection{Admissibility of the Relaxation}
\begin{lemma}
	\label{lemma:real-admissible}
	The prediction rule $\sQ=(\cQ_1,\cdots,\cQ_T)$ given by \cref{alg:real-valued} is admissible with respect to the relaxation defined in \cref{eq:relaxation-real}.
\end{lemma}
\begin{proof}
	Using the language of regularized Rademacher complexity, the above relaxation can be written as
	\begin{align}
		\rel_T(\cH|s_{1:t})=2G\Ex_{V^{(t)}\simiid \unif(\cX)}\left[
		\mathfrak{R}(-{L^{\mathrm{r}}(\cdot,s_{1:t})},V^{(t)})
		\right]+2G\beta(T-t),
	\end{align}
	where $L^{\mathrm{r}}(\cdot,s_{1:t})=\sum_{i=1}^{t-1} l^{\mathrm{r}}(h(x_i),y_i)$. When $t=T$, the relaxation becomes
	\begin{align}
		\rel_T(\cH|s_{1:T})=-2G L^{\mathrm{r}}(h,s_{1:T})=-\inf_{h\in\cH}\sum_{i=1}^T l(h(x_i),y_i),
	\end{align}
	thus it satisfies the second condition of \Cref{def:admissibility}. For the first condition, we need to verify
	\begin{align}
		\sup_{\cD_t\in\mathfrak{D}_t}\Ex_{x_t\sim \cD_t}\sup_{y_t\in\cY}\!\left\{\!
		\Ex_{\hy_t\sim\cQ_t}\![l(\hy_t,y_t)]\!+\!\rel_T(\cH\mid s_{1:t-1}\!\cup\!(x_t,y_t))\!
		\right\}
		\!\le\! \rel_{T}(\cH\mid s_{1:t-1}). \label{eq:goal-admissible-real}
	\end{align}
	
	We first upper bound the LHS of \cref{eq:goal-admissible-real} by matching the randomness in $V^{(t)}$ and applying Jensen's inequality to the supremum function. This gives us
	\begin{align}
		&\sup_{\cD_t\in\mathfrak{D}_t}\Ex_{x_t\sim \cD_t}\sup_{y_t\in\cY}\!\left\{\!
		\Ex_{\hy_t\sim\cQ_t}\![l(\hy_t,y_t)]\!+\!\rel_T(\cH\mid s_{1:t-1}\!\cup\!(x_t,y_t))\!
		\right\}\\
		=&\sup_{\cD_t\in\mathfrak{D}_t}\Ex_{x_t\sim \cD_t}\sup_{y_t\in\cY}\Ex_{V^{(t)}\simiid \unif(\cX)}\left[
		\Ex_{\hy_t\sim\cQ_t(V^{(t)})}\left[l(\hy_t,y_t)\right]+2G\cdot \mathfrak{R}(-{L^{\mathrm{r}}(\cdot,s_{1:t})},V^{(t)})
		\right]+2G\beta(T-t)
        \label[Ineq]{eq:matching-randomness}
        \\
		\le&\sup_{\cD_t\in\mathfrak{D}_t}\Ex_{x_t\sim \cD_t}\Ex_{V^{(t)}\simiid \unif(\cX)}\left[\sup_{y_t\in\cY}\left\{
		\Ex_{\hy_t\sim\cQ_t(V^{(t)})}\left[l(\hy_t,y_t)\right]+2G\cdot \mathfrak{R}(-{L^{\mathrm{r}}(\cdot,s_{1:t})},V^{(t)})\right\}
		\right]+2G\beta(T-t)
		\label{eq:real-tmp1}
	\end{align}

	For every fixed input $x_t$ and hint set $V^{(t)}$, our prediction rule in \cref{eq:prediction-rule-real}, which we denote with $\cQ_t(V^{(t)})$, is the same as the transductive prediction rule in \cite[Equation (25)]{rakhlin2012relax}, with $V^{(t)}$ being the set of unlabeled future instances and $s_{1:t-1}$ being the historical data with labels.

	\newcommand{\seq}{\mathscr{X}}
	{According to \cite[Lemma 12]{rakhlin2012relax}, for all input $x_t$ and unlabeled sequence $\seq$ (which plays the role of $x_{t+1:T}$), the decision rule $\cQ_t(\seq)$ satisfies}
	\begin{align}
	    &\sup_{y_t\in\cY}\left\{
	    \Ex_{\hy_t\sim\cQ_t(\seq)}[l(\hy_t,y_t)]
	    +2G\Ex_{\cE}\left[\sup_{h\in\cH}\left\{
	    \sum_{x\in\seq} \epsilon_x h(x)-L^{\text{r}}(h,s_{1:t-1}\cup(x_t,y_t))
	    \right\}\right]
	    \right\}\\
	    &\qquad\qquad\qquad\le 2G\Ex_{\cE}\left[\sup_{h\in\cH}\left\{
	    \sum_{x\in\seq\cup\{x_t\}} \epsilon_x h(x)-L^{\text{r}}(h,s_{1:t-1})
	    \right\}\right].
	\end{align}
	
	{Therefore, if we choose the sequence $\seq$ to be $V^{(t)}$, we obtain the following inequality which is written in the language of regularized Rademacher complexity:}
	\begin{align}
		\sup_{y_t\in\cY}\left\{
		\Ex_{\hy_t\sim\cQ_t(V^{(t)})}\left[l(\hy_t,y_t)\right]+2G\cdot\mathfrak{R}(-L^{\mathrm{r}}(\cdot,s_{1:t}),V^{(t)})
		\right\}
		\le 2G\cdot\mathfrak{R}(-L^{\mathrm{r}}(\cdot,s_{1:t-1}),V^{(t)}\cup\{x_t\}).
	\end{align}
	By adding the expectations over $V^{(t)}$ and $x_t$ on both sides, we obtain the following upper bound:
	\begin{align}
		\eqref{eq:real-tmp1}\le&\sup_{\cD_t\in\mathfrak{D}_t}\Ex_{x_t\sim \cD_t}\Ex_{V^{(t)}\simiid \unif(\cX)}\left[
		2G\cdot\mathfrak{R}(-L^{\mathrm{r}}(\cdot,s_{1:t-1}),V^{(t)}\cup\{x_t\})
		\right]+2G\beta(T-t)\\
		\le &\Ex_{V^{(t)}\simiid \unif(\cX)}\left[
		\sup_{\cD_t\in\mathfrak{D}_t}\Ex_{x_t\sim \cD_t}
		2G\cdot\mathfrak{R}(-L^{\mathrm{r}}(\cdot,s_{1:t-1}),V^{(t)}\cup\{x_t\})
		\right]+2G\beta(T-t).
	\end{align}
	According to \Cref{lemma:sup-to-distribution}, we can replace the $x_t$ sampled from the worst-case smooth distribution by $Z_t$ sampled independently from the uniform distribution, {with the extra cost $\beta$}. This gives
	\begin{align}
		\eqref{eq:real-tmp1}\le&\Ex_{{V^{(t)}\simiid \unif(\cX)}}\Ex_{Z_t\simiid\unif(\cX)}\left[2G\cdot\left(
		\mathfrak{R}(-L^{\mathrm{r}}(\cdot,s_{1:t-1}),V^{(t)}\cup Z_t)+\beta\right)
		\right]+2G\beta(T-t)\\
		=&2G\Ex_{V^{(t)}\simiid \unif(\cX)}\left[
		\mathfrak{R}(-{L^{\mathrm{r}}(\cdot,s_{1:t-1})},V^{(t-1)})
		\right]+2G\beta(T-t+1)
        \label{eq:same-distribution}\\
        =&\rel_T(\cH|s_{1:t-1}),
	\end{align}
	which is precisely the RHS of \cref{eq:goal-admissible-real}.
\end{proof}

\begin{lemma}[Replacing Supremum by Expectation]
    \label{lemma:sup-to-distribution}
    For any $V^{(t)}\in\cX^{K(T-t)}$, there exists a set of $K$ variables $Z_t=\{z_{t,k}\}_{k\in[K]}$, such that
    \begin{align}
        \sup_{\cD_t^{\cX}\in\Delta_\sigma(\cX)}\Ex_{{x_t\sim \cD_t^{\cX}}}\left[\mathfrak{R}(-L^{\mathrm{r}}(\cdot,s_{1:t-1}),V^{(t)}\cup\{x_t\})\right]\le \Ex_{Z_t\simiid\unif(\cX)}\left[
        \mathfrak{R}(-L^{\mathrm{r}}(\cdot,s_{1:t-1}),V^{(t)}\cup Z_t)\right]+\beta.
    \end{align}
    \end{lemma}
         
         \begin{proof}
    To establish the monotonicity property, we need to show that the random instance $x_t$ drawn from a smooth distribution belongs to a set of uniform i.i.d. hints with high probability.
    This is where the coupling lemma comes in.
    For the smooth distribution $\cD_t\in\Delta_{\sigma}(\cX)$ that achieves the supremum (assume the supremum is achievable),
    \Cref{lemma:coupling} shows the existence of a coupling $\Pi$ on $(x_t,z_{t,1},\cdots,z_{t,K})$ such that $x_t$ is distributed according to $\cD_t^{\cX}$ and $Z_t=\{z_{t,k}\}_{k\in[K]}$ are uniformly and independently distributed.
    We thus have
    \begin{align}
        \sup_{\cD_t\in\Delta_\sigma(\cX)}\Ex_{{x_t\sim \cD_t}}\left[\mathfrak{R}(-L^{\mathrm{r}}(\cdot,s_{1:t-1}),V^{(t)}\cup\{x_t\})\right]=\Ex_{\Pi}\left[\mathfrak{R}(-L^{\mathrm{r}}(\cdot,s_{1:t-1}),V^{(t)}\cup\{x_t\})\right].
        \label{eq:medium}
    \end{align}
    
    This joint distribution $\Pi$ has the property that event the $E_t\overset{\text{def}}{=}\{x_t\in Z_t\}$ happens with high probability. 
    We now upper bound the expected value by conditioning on $E_t$ and $\bar{E_t}$ respectively.
    
    Conditioned on $E_t$, 
    we apply the monotonicity of regularized Rademacher complexity (\Cref{lem:monotonicity}) recursively and obtain
    \begin{align}
        \mathfrak{R}(-L^{\mathrm{r}}(\cdot,s_{1:t-1}),V^{(t)}\cup\{x_t\})
        \le\mathfrak{R}(-L^{\mathrm{r}}(\cdot,s_{1:t-1}),V^{(t)}\cup Z_t).
        \label{ineq:omit-2}
    \end{align}
    
    Conditioned on $\bar{E_t}$, we skirt the monotonicity issue by directly using upper and lower bounds on the regularized Rademacher complexity.
    To be more precise,
    we use \Cref{lem:bounded-relaxations} {in \cref{appendix:lemma-bounds-relaxation}} to show that
    \begin{align}
        \mathfrak{R}(-L^{\mathrm{r}}(\cdot,s_{1:t-1}),V^{(t)}\cup\{x_t\})\le &
        TK\le TK+\left(
            \mathfrak{R}(-L^{\mathrm{r}}(\cdot,s_{1:t-1}),V^{(t)}\cup Z_t)
            +{T}\right)\\
            \le& \mathfrak{R}(-L^{\mathrm{r}}(\cdot,s_{1:t-1}),V^{(t)}\cup Z_t)
            +2T K.
            \label{ineq:zt-bar}
    \end{align}
    
    Finally, we expand the right hand side of \cref{eq:medium} by conditioning on $E_t$ and $\bar{E_t}$ respectively. 
    Putting \cref{ineq:omit-2,ineq:zt-bar} together, we obtain
    
    \begin{align}
        \text{\eqref{eq:medium}}=&
        \Pr[E_t]\cdot\Ex_\Pi\left[\left.\mathfrak{R}(-L^{\mathrm{r}}(\cdot,s_{1:t-1}),V^{(t)}\cup\{x_t\})\right| E_t\right]
        \\
        &\qquad\qquad\qquad+\Pr[\bar{E_t}]\cdot\Ex_\Pi\left[\left. \mathfrak{R}(-L^{\mathrm{r}}(\cdot,s_{1:t-1}),V^{(t)}\cup\{x_t\})\right| \bar{E_t}\right]\\
        \le&\Pr[E_t]\cdot \Ex_\Pi\left[\left.
        \mathfrak{R}(-L^{\mathrm{r}}(\cdot,s_{1:t-1}),V^{(t)}\cup Z_t)
        \right|E_t\right]\\
        &\qquad\qquad\qquad+\Pr[\bar{E_t}]\cdot
        \Ex_\Pi\left[\left.
        \mathfrak{R}(-L^{\mathrm{r}}(\cdot,s_{1:t-1}),V^{(t)}\cup Z_t)+2 T K
        \right|\bar{E_t}\right]\\
        =&\Ex_\Pi\left[
        \mathfrak{R}(-L^{\mathrm{r}}(\cdot,s_{1:t-1}),V^{(t)}\cup Z_t)\right]
        +\Pr[\bar{E_t}]\cdot 2TK.
         \end{align}
        
        Since $\Pi$ has uniform marginal distribution on $ Z_t$, and that $\Pr[\bar{E_t}]\cdot 2TK\le(1-\sigma)^K\cdot 2TK\le\beta$, we further obtain
        \begin{align}
            \eqref{eq:medium}\le&\Ex_{Z_t\simiid\unif(\cX)}\left[
        \mathfrak{R}(-L^{\mathrm{r}}(\cdot,s_{1:t-1}),V^{(t)}\cup Z_t)\right]+\beta,
        \end{align}
        thus completes the proof.
         \end{proof}

\subsection{Upper and Lower Bounds on the Relaxation}
         \label{appendix:lemma-bounds-relaxation}
         \begin{lemma}[Upper and Lower Bounds on the Relaxation]
         	\label{lem:bounded-relaxations}
         	For all $t\in[T]$, all sequence $s_{1:T}$, and all instance set $Z$ of size no larger than $(T-t)K$,
         	\begin{align}
         		-\frac{T}{2}\le
         		\mathfrak{R}(-L^{\mathrm{r}}(\cdot,s_{1:t}),Z)
         		\le T K.
         	\end{align}
         \end{lemma}
         \begin{proof}
         	By convexity of the supremum, 
         	\begin{align}
         		\mathfrak{R}(-L_t,Z)
         		&=\Ex_{{\cE\overset{\text{iid}}{\sim}\unif(\cY)}}\left[\sup_{h\in\cH}\left\{
         		\sum_{i=1}^{I} \epsilon_i h(z_i)-L^{\mathrm{r}}(h,s_{1:t})
         		\right\}\right]
         		\ge\sup_{h\in\cH}\left\{
         		\Ex_{{\cE\overset{\text{iid}}{\sim}\unif(\cY)}}\left[\sum_{i=1}^{I} \epsilon_i h(z_i)-L^{\mathrm{r}}(h,s_{1:t})\right]
         		\right\}\\
         		&=\sup_{h\in\cH}\sum_{i=1}^t \underbrace{(-l^{\mathrm{r}}(h(x_t),y_t))}_{\ge -1}
         		\ge-{T}.
         	\end{align}
         	
         	For the upper bound, we notice that $\forall \cE,h$,
         	\begin{align}
         		\sum_{i=1}^{I} \epsilon_i h(z_i)-L^{\mathrm{r}}(h,s_{1:t})\le
         		I+t\le (T-t)K+t\le TK.
         	\end{align} 
         	So the $\mathfrak{R}(-L^{\mathrm{r}}(\cdot,s_{1:t}),Z)$ also has an upper bound of $TK$.
         \end{proof}

\subsection{Remark on the Requirement of Fresh Dataset}
    \label{sec:comment-on-fresh-dataset}
    In order to beat the adaptive adversary, the learner needs to sample fresh random hints in each round. Otherwise, the adversary can enforce high regret by correlating future labels with the history. More precisely,
	we will see that the \emph{matching randomness} argument in \cref{eq:matching-randomness} uses
    {the crucial fact that $V^{(t)}$ is a \emph{fresh} dataset that is uniformly distributed independent of the interactions in the past. 
    If $V^{(t)}$ is reused, then the adaptive adversary has the power to correlate $s_{1:t-1}$ with $V^{(t)}$ such that $V^{(t)}$ is no longer unbiased conditioned on the history.
    In this case, the algorithm fails to mimic the randomization in the relaxation, and the matching-randomness argument breaks down.}

	Another important property of the fresh self-generated hints $v^{(t)}_{i,k}$ is that they are identically distributed with the real hints $z_{i,k}$ in the coupling.
	{Nevertheless, the analysis has to unite the fact that the learner can only access $v^{(t)}_{i,k}$s, and the monotonicity property (\cref{lemma:sup-to-distribution}) is based on $z_{i,k}$.}
	This point is subtle because it is impossible for the self-generated hints to really tell the future (i.e., ensure $x_{t}\in \{v_{t,k}^{(t-1)}\}_{k\in[K]}$), since they are not controlled by the coupling $\Pi$.
            This issue is taken care of by \cref{eq:same-distribution}. We can see that
            it is sufficient for the uncoupled hints in $V^{(t-1)}$ to resemble the coupled hints $Z_t$ at \emph{distribution} level.
			This distributional resemblence is not achievable if $V^{(t-1)}$ were not independent with the past.

\section{Smoothed Online Learning with Binary-valued Functions}
\subsection{Information Theoretic Lemmas} \label{app:auxlemma-chi2}
For two probability distributions $P$ and $Q$ over the same domain $\cX$, let $\chi^2(P,Q) = \sum_{x\in \cX} \frac{P(x)^2}{Q(x)} - 1$ be the $\chi^2$-divergence. The following lemma upper bounds the TV distance by the $\chi^2$-divergence; a proof could be found in \cite[Chapter 2]{Tsybakov2009}. 
\begin{lemma}[From TV to $\chi^2$]\label{lemma:TV_chi2}
The following relations hold: 
\begin{align}\label{eq:TV_chi2}
\text{\rm TV}(P,Q) &\le \sqrt{\frac{1}{2}\log(1+\chi^2(Q,P))} \le \sqrt{\frac{\chi^2(Q,P)}{2}}.
\end{align}
\end{lemma}

The following statement is the well-known Ingster's $\chi^2$ method, and we refer to the excellent book \cite{ingster2003nonparametric} for a general treatment. 
\begin{lemma}[Ingster's $\chi^2$ method]\label{lemma:chi2_mixture}
For a mixture distribution $\Ex_{\theta\sim \pi}[Q_\theta]$ and a generic distribution $P$, the following identity holds: 
\begin{align}\label{eq:chi2_mixture}
\chi^2\left(\Ex_{\theta\sim \pi}[Q_\theta], P\right) &= \Ex_{\theta,\theta'\sim \pi}\left[\Ex_{x\sim P}\left(\frac{Q_\theta(x)Q_{\theta'}(x)}{P(x)^2}\right) \right] - 1, 
\end{align}
where $\theta'$ is an independent copy of $\theta$. 
\end{lemma}

\subsection{
Proof of \Cref{lemma:stability_admiss}} \label{sec:proof_stability_admiss}
Using the definitions of $\cQ_t, r^t$, and $\mathcal{P}^{t}$, the following chain of inequalities holds for any fixed $s_t$: 
\begin{align}
&\Ex_{h_t\sim \cQ_t}[L(h_t, s_t)] + \rel_T(\cH\mid s_{1:t})\\
&\stepa{=} \Ex_{\mathcal{P}^{t-1}}[L(\opt_{\cH,l}(r^{t-1}), s_t)] - \Ex_{R^{(t+1)}}\left[\sum_{i=1}^{N^{(t+1)}} L(\opt_{\cH,l}(r^{t}), \widetilde{s}_i^{(t+1)}) + \sum_{\tau=1}^t L(\opt_{\cH,l}(r^{t}), s_\tau) \right] + \eta(T-t) \\
&= \Ex_{\mathcal{P}^{t-1}}[L(\opt_{\cH,l}(r^{t-1}), s_t)] - \Ex_{\mathcal{P}^{t}}[L(\opt_{\cH,l}(r^{t}), s_t)] + \eta(T-t) \\
&\qquad - \Ex_{R^{(t+1)}}\left[\sum_{i=1}^{N^{(t+1)}} L(\opt_{\cH,l}(r^{t}), \widetilde{s}_i^{(t+1)}) + \sum_{\tau=1}^{t-1} L(\opt_{\cH,l}(r^{t}), s_\tau) \right] \\
&\le \Ex_{h_t\sim \cQ_{t}}[L(h_t, s_t)] - \Ex_{h_{t+1}\sim \cQ_{t+1}}[L(h_{t+1}, s_t)] + \eta(T-t)  + \Ex_{R^{(t+1)}}\left[\sup_{h\in \cH} \left(-\sum_{i=1}^{N^{(t+1)}} L(h, \widetilde{s}_i^{(t+1)}) - \sum_{\tau=1}^{t-1} L(h, s_\tau)\right) \right] \\
&\stepb{=} \Ex_{h_t\sim \cQ_{t}}[L(h_t, s_t)] - \Ex_{h_{t+1}\sim \cQ_{t+1}}[L(h_{t+1}, s_t)] + \eta(T-t) + \Ex_{R^{(t)}}\left[\sup_{h\in \cH} \left(-\sum_{i=1}^{N^{(t)}} L(h, \widetilde{s}_i^{(t)}) - \sum_{\tau=1}^{t-1} L(h, s_\tau)\right) \right] \\
& = \Ex_{h_t\sim \cQ_{t}}[L(h_t, s_t)] - \Ex_{h_{t+1}\sim \cQ_{t+1}}[L(h_{t+1}, s_t)] - \eta + \rel_T(\cH \mid s_{1:t-1}), 
\end{align}
where (a) uses the definition of $\opt_{\cH,l}(r^t)$, and (b) is due to the fact that $R^{(t+1)}$ is an independent copy of $R^{(t)}$ conditioned on $\{s_\tau\}_{\tau<t}$. This implies the first inequality of \Cref{lemma:stability_admiss}. 

For the second inequality, we further take the expectation with respect to $s_t\sim \cD_t$, and note that $\cQ_t$ and $\rel_T(\cH\mid s_{1:t-1})$ are independent of $s_t$, while $\cQ_{t+1}$ depends on $s_t$: 
\begin{align}
&\Ex_{s_t\sim \cD_t}\left(\Ex_{h_t\sim \cQ_t}[L(h_t, s_t)] + \rel_T(\cH\mid s_{1:t})\right) -  \rel_T(\cH\mid s_{1:t-1}) \\
&\le \Ex_{s_t\sim \cD_t}\Ex_{h_t\sim \cQ_{t}}[L(h_t, s_t)] - \Ex_{s_t\sim \cD_t}\Ex_{h_{t+1}\sim \cQ_{t+1}}[L(h_{t+1}, s_t)] - \eta \\
&\le \Ex_{s_t\sim \cD_t}\Ex_{h_t\sim \cQ_{t}}[L(h_t, s_t)] - \Ex_{s_t, s_t'\sim \cD_t}\Ex_{h_{t+1}\sim \cQ_{t+1}}[L(h_{t+1}, s_t')] \\
&\qquad + \Ex_{s_t, s_t'\sim \cD_t}\Ex_{h_{t+1}\sim \cQ_{t+1}}[L(h_{t+1}, s_t')] - \Ex_{s_t\sim \cD_t}\Ex_{h_{t+1}\sim \cQ_{t+1}}[L(h_{t+1}, s_t)] - \eta \\
&\stepc{=} \Ex_{s_t'\sim \cD_t}\Ex_{h_t\sim \cQ_{t}}[L(h_t, s_t')] - \Ex_{s_t'\sim \cD_t}\Ex_{h_{t+1}\sim \Ex_{s_t\sim \cD_t}[\cQ_{t+1}]}[L(h_{t+1}, s_t')] \\
&\qquad + \Ex_{s_t, s_t'\sim \cD_t}\Ex_{h_{t+1}\sim \cQ_{t+1}}[L(h_{t+1}, s_t')] - \Ex_{s_t\sim \cD_t}\Ex_{h_{t+1}\sim \cQ_{t+1}}[L(h_{t+1}, s_t)] - \eta \\
&\stepd{\le} \text{TV}(\cQ_t, \Ex_{s_t\sim \cD_t}[\cQ_{t+1}]) + \Ex_{s_t, s_t' \sim \cD_t; R^{(t+1)}}\left[L(h_{t+1}, s_t') - L(h_{t+1},s_t) \right] - \eta, 
\end{align}
where (c) follows from the independence of $h_t\sim \cQ_t$ and $(s_t, s_t')$, and (d) is due to $|\Ex_{X\sim P}[f(X)] - \Ex_{X\sim Q}[f(X)]|\le \text{TV}(P,Q)$ for every measurable function $f$ with $\|f\|_\infty \le 1$. 

\subsection{Proof of \Cref{lemma:coupling_strong}}
The proof is essentially similar to \cite[Lemma 12]{BlockDGR22}, and we include it here for completeness. For each $i\in [m]$, compute the value $p_i = \sigma\frac{dP}{dQ}(X_i)$, which lies in $[0,1]$ due to the likelihood ratio upper bound. Now we draw an independent Bernoulli random variable $Y_i\sim \text{Bern}(p_i)$, and define the random index $I$ and success event $E$ as follows: 
\begin{align}
    E &\triangleq \cup_{i=1}^m \{Y_i = 1\}, \\
    I &\triangleq \text{a uniformly random element of } \{i\in [m]: Y_i = 1\}.
\end{align}

Note that $Y_1, \cdots, Y_m$ are mutually independent, and for each $i\in [m]$, 
\begin{align}
    \Pr[Y_i = 1] = \Ex_{X_i\sim Q}[p_i] = \Ex_{X_i\sim Q}\left[\sigma \frac{dP}{dQ}(X_i)\right] = \sigma, 
\end{align}
we conclude that $\Pr[E]=1-(1-\sigma)^m$. For the second statement, we denote by $r_i$ the external randomness used in drawing $Y_i\sim \text{Bern}(p_i)$, and by $r$ the external randomness used in the definition of $I$. Then for any measurable set $A \subseteq \cX$, 
\begin{align}
    &\Pr[ X_I\in A \mid E, X_{\backslash I} ] \\
    &= \sum_{i, r_{\backslash i}, r} \Pr[ X_I\in A \mid E, X_{\backslash I}, I = i, r_{\backslash i}, r]\cdot \Pr[I=i, r_{\backslash i}, r \mid E, X_{\backslash I}] \\
    &= \sum_{i, r_{\backslash i}, r} \Pr[ X_i\in A \mid E, X_{\backslash i}, I = i, r_{\backslash i}, r]\cdot \Pr[I=i, r_{\backslash i}, r \mid E, X_{\backslash I}] \\
    &\stepa{=} \sum_{i, r_{\backslash i}, r} \Pr[ X_i\in A \mid Y_i = 1, X_{\backslash i}, r_{\backslash i}, r]\cdot \Pr[I=i, r_{\backslash i}, r \mid E, X_{\backslash I}] \\
    &\stepb{=} \sum_{i, r_{\backslash i}, r} \Pr[ X_i\in A \mid Y_i = 1]\cdot \Pr[I=i, r_{\backslash i}, r \mid E, X_{\backslash I}] \\
    &\stepc{=} \sum_{i, r_{\backslash i}, r} P(A)\cdot \Pr[I=i, r_{\backslash i}, r \mid E, X_{\backslash I}] \\
    &= P(A),
\end{align}
where (a) is due to the event $\{E, I=i, X_{\backslash i}, r_{\backslash i}, r \}$ is the same as $\{Y_i=1, X_{\backslash i}, r_{\backslash i}, r \}$ as long as the former event $\{E, I=i, X_{\backslash i}, r_{\backslash i}, r \}$ is non-empty (note that empty events do not contribute to the sum), (b) follows from the mutual independence of $(X_i, r_i, Y_i)_{i\in [m]}$ and $r$, (c) is due to
\begin{align}
\Pr[ X_i\in A \mid Y_i = 1] = \frac{\Pr[X_i\in A, Y_i = 1]}{\Pr[Y_i = 1]} = \frac{1}{\sigma}\Ex_{X_i\sim Q}\left[\mathbf{1}(X_i\in A) \sigma\frac{dP}{dQ}(X_i) \right] = P(A). 
\end{align}
The above identity shows that the conditional distribution of $X_I$ conditioned on $(E, X_{\backslash I})$ is always $P$, as desired.

\subsection{Proof of \Cref{lemma:stability_admiss_modified}}
The analysis is similar to the proof of \Cref{lemma:stability_admiss}. In fact, an intermediate step of \Cref{lemma:stability_admiss} gives
\begin{align}
\Ex_{h_t\sim \cQ_t}[L(h_t, s_t)] + \rel_T(\cH\mid s_{1:t}) - \rel_T(\cH \mid s_{1:t-1}) \le \Ex_{h_t\sim \cQ_{t}}[L(h_t, s_t)] - \Ex_{h_{t+1}\sim \cQ_{t+1}}[L(h_{t+1}, s_t)] - \eta. 
\end{align}

Now using $|\Ex_{X\sim P}[f(X)] - \Ex_{X\sim Q}[f(X)]|\le \text{TV}(P,Q)$ for every measurable function $f$ with $\|f\|_\infty \le 1$, the RHS is further upper bounded by $\text{TV}(\cQ_t, \cQ_{t+1}) - \eta$. The proof of \Cref{lemma:stability_admiss_modified} is completed by taking the expectation over $s_t\sim \cD_t$.  
\section{Proof of Lower Bounds (\cref{thm:lower_bound} and \cref{thm:comp_lower_bound})}

\subsection{Proof of \cref{thm:lower_bound}}

This section proves the regret lower bounds for \cref{alg:FTPL} and \cref{alg:transductive_k_hints} stated in \cref{thm:lower_bound}. We split the analysis into two subsections, and in each subsection we prove a large regret both when the sample size parameter $n$ is large and small. 

\subsubsection{Lower Bound Analysis for \cref{alg:FTPL}}
We shall only prove the regret lower bound $\Omega(\sqrt{dT\sigma^{-1/2}})$ under the assumption $\sigma \ge \max\{d/|\cX|, (d/T)^2\}$, for a smaller $\sigma$ only makes the worst-case regret larger, and the other lower bounds follow from this case by taking $\sigma = d/|\cX|$ and $\sigma = (d/T)^2$, respectively. We split the analysis into two cases depending on the choice of parameter $n$. 

\paragraph{\textbf{Case I: Large $n$.}}\label{subsec:proof_FTPL_large_n}
When $n$ is large, or more specifically, when $n\ge T/\sqrt{\sigma}$, consider the behavior of \cref{alg:FTPL} on the following instance. Consider any domain $\cX$ where $|\cX|$ is an integral multiple of $d$, and partition $\cX = \cup_{j=1}^d \cX_j$ into $d$ sets $\{\cX_j\}_{j\in [d]}$ with an equal size. Consider the following hypothesis class: 
\begin{align}
\cH = \{h: \cX\to \{\pm 1\} \mid h \text{ is a constant on } \cX_i, \forall i\in [d] \}. 
\end{align}
Clearly $\cH$ has VC dimension $d$. The adversary chooses a hypothesis $h^\star\in \cH$ uniformly at random, and sets $x_t$ to be uniformly distributed on $\cX$. As for the label $y_t$, the adversary sets $y_t = h^\star(x_t)$. This adversary is $1$-smooth, and the best expert in $\cH$ incurs a zero loss under this realizable setting. We claim that for each of the first $\min\{T, c\sqrt{nd}\}$ time steps, for an absolute constant $c>0$ sufficiently small, \cref{alg:FTPL} makes a mistake with $\Omega(1)$ probability. Summing over these steps, the expected regret of \cref{alg:FTPL} is then $\Omega(\min\{T, c\sqrt{nd}\})$, which gives Theorem \ref{thm:lower_bound} by our assumption $n\ge T/\sqrt{\sigma}$.

To prove this claim, we need the following lemma. 
\begin{lemma}[Minimum Error on Hallucinated Samples]\label{lemma:min_error}
For $N\sim \text{\rm Poi}(n)$ hallucinated samples $(x_1,y_1), \cdots, (x_N, y_N)$, if $n\ge d$, it holds that
\begin{align}
    \mathbb{P}\left(\sum_{i=1}^N y_i\cdot \mathbf{1}(x_i\in \cX_j) \ge \sqrt{\frac{n}{d}} \right) = \Omega(1), \qquad \forall j\in [d].  
\end{align}
\end{lemma}
\begin{proof}
For $j\in [d]$, let $n_{j,+}, n_{j,-}$ denote the number of hallucinate samples $(x_i, y_i)$ with $x_i \in \cX_j$ and $y_i = \pm 1$, respectively. By the Poisson subsampling property, $\{n_{j,\pm}\}_{j\in [d]}$ are mutually independent $\text{Poi}(n/(2d))$ random variables. By definition of $\cH$, we have
\begin{align}
    n_{j,+} - n_{j,-} = \sum_{i=1}^N y_i\cdot \mathbf{1}(x_i\in \cX_j). 
\end{align}
Consequently, the quantity of interest is $n_{j,+} - n_{j,-}$. As $n/d\ge 1$, by the Poisson tail property, both events $n_{j,+} \ge n/(2d) + \sqrt{n/d}/2$ and $n_{j,-} \le n/(2d) - \sqrt{n/d}/2$ happen with $\Omega(1)$ probability, and their independence gives the claimed result. 
\end{proof} 

Since $(d/T)^2\le \sigma \le 1$, we have $T\ge d$ and thus $n\ge T/\sqrt{\sigma}\ge d$, the premise of \Cref{lemma:min_error} holds. Consequently, at each time step $t\le \min\{T,c\sqrt{nd}\}$ with $x_t\in \cX_j$, with $\Omega(1)$ probability there are at least $\sqrt{n/d}$ net positive labels in the hallucinated samples, while the learner has only observed at most $\alpha c\sqrt{n/d}$ labels in the history with probability at least $1-1/\alpha$, by Markov's inequality. By choosing constants $c>0$ small and $\alpha>0$ large, the perturbed leader will predict $+1$ depending only on the hallucination, and this prediction is independent of the choice of $h^\star$ and thus incurs an error with probability $1/2$. This proves the claim that before time $\min\{T, c\sqrt{nd}\}$, there is always $\Omega(1)$ probability of error. 

\paragraph{\textbf{Case II: Small $n$.}}\label{subsec:proof_FTPL_small_n}
Now we turn to the scenario where $n< T/\sqrt{\sigma}$. Consider the following learning instance: choose $\cX_0\subseteq \cX$ with $|\cX_0| = \sigma |\cX|\ge d$, the adversary always chooses $x_t \sim \unif(\cX_0)$, which is $\sigma$-smooth. Assuming that $|\cX_0|$ is an integral multiple of $d$, we partition $\cX_0 = \cup_{j=1}^d \cX_j$ into $d$ subsets with equal size. Condition on each $\cX_j$, consider an alternating label sequence: 
\begin{align}
    (y_t: x_t\in \cX_j)_{t=1}^T = (+1, -1, +1, -1, \cdots). 
\end{align}
The hypothesis class $\cH$ consists of $2^d$ functions: 
\begin{align}
\cH = \{h: \cX\to \{\pm 1\} \mid h \text{ is a constant on } \cX_j, \forall j\in [d], \text{ and } h(x) \equiv 1, \forall x\in \cX \backslash \cX_0 \}. 
\end{align}
Clearly $\cH$ has VC dimension $d$, and the best hypothesis in $\cH$ incurs a cumulative loss $T/2$. 

Now we examine the performance of \cref{alg:FTPL}. Let $r_j$ be the difference between the number of $+1$ and $-1$ labels in the hallucinated samples with feature in $\cX_j$, similar to the proof of \Cref{lemma:min_error} we have $r_j = n_{j,+} - n_{j,-}$ for independent Poisson random variables $n_{j,+}, n_{j,-}\sim \text{Poi}(n\sigma/2d)$. Suppose that ties are broken by always predicting $-1$ when calling the ERM oracle, we observe that \cref{alg:FTPL} always makes a mistake when $x\in \cX_j$ and $r_j=0$ -- this is the same counterexample where Follow-The-Leader (FTL) makes a mistake at every step. Moreover, when $r_j\neq 0$, Algorithm \ref{alg:FTPL} makes $T/2$ mistakes, same as the best expert in $\cH$. Consequently, the expected regret of \cref{alg:FTPL} is at least $T\cdot \mathbb{P}(r_j=0)$, where
\begin{align}
    \mathbb{P}(r_j=0) &= \Ex_{N\sim \text{Poi}(n\sigma/d)}\left[\mathbb{P}\left(\text{Bin}(N,\frac{1}{2})=\frac{N}{2}\right)\right] = \Ex_{N\sim \text{Poi}(n\sigma/d)}\left[\Omega\left(\frac{\mathbf{1}(N\text{ is even})}{\sqrt{N+1}}\right)\right] \\
    &\stepa{=} \Omega\left( \frac{\mathbb{P}_{N\sim \text{Poi}(n\sigma/d)}(N\text{ is even}) }{\sqrt{n\sigma/d + 1}} \right) \stepb{=} \Omega\left(\frac{1}{\sqrt{n\sigma/d+1}}\right) = \Omega\left(\min\left\{1,\sqrt{\frac{d}{n\sigma}}\right\}\right).
\end{align}
In the above display, (a) follows from the conditional Jensen's inequality, and (b) is due to
\begin{align}
\mathbb{P}_{N\sim \text{Poi}(\lambda)}(N\text{ is even}) = \sum_{k=0}^\infty e^{-\lambda}\frac{\lambda^{2k}}{(2k)!} = e^{-\lambda}\cdot \frac{e^\lambda + e^{-\lambda}}{2} \ge \frac{1}{2}. 
\end{align}
This leads to the claimed regret lower bound in Theorem \ref{thm:lower_bound}. 

\subsubsection{Lower Bound Analysis for \cref{alg:transductive_k_hints}}
Similar to the lower bound analysis for \cref{alg:FTPL}, we also split into the cases where $n$ is large and $n$ is small, respectively. Recall that for \cref{alg:transductive_k_hints}, the parameter $n$ is the number of random draws from $K$ hints at each future time. 

\paragraph{\textbf{Case I: Large $n$.}}
We first focus on the case where $n\ge \sqrt{K}$. Consider the following learning instance: the domain $\cX$ is $[dK]$, and we partition $\cX$ into $\cup_{j=1}^d \cX_j$ each of size $K$. At each time, the subsets $\cX_j$ are given as the hint cyclically. Consider the same construction of $\cH$ and the adversary in Section \ref{subsec:proof_FTPL_large_n}, except that each $x_t$ is now uniformly distributed in the $K$-hint set. 

The regret analysis is essentially the same as Section \ref{subsec:proof_FTPL_large_n}. For every $t\le T/2$, the learner in \cref{alg:transductive_k_hints} essentially generates $(T-t)n\ge T\sqrt{K}/4$ uniformly random samples (with replacement) in $\cX$. A similar analysis to \Cref{lemma:min_error} shows that for each $j\in [d]$, with $\Omega(1)$ probability there are $\Omega(\sqrt{K^{1/2}T/d})$ more $+1$ labels than $-1$ labels within $\cX_j$ in the hallucinated samples. Consequently, for $t\le \min\{T/2, \Omega((dT)^{1/2}K^{1/4}\}$, two calls of the ERM oracle in \cref{alg:transductive_k_hints} will return the same hypothesis, and the learner's prediction is always $+1$. Similar to Section \ref{subsec:proof_FTPL_large_n}, these time steps lead to an $\Omega(\min\{T, (dT)^{1/2}K^{1/4}\})$ regret. 

\paragraph{\textbf{Case II: Small $n$.}}
Next we turn to the case where $n<\sqrt{K}$. Again, we construct $\cX$ to be the disjoint union of $d$ sets $\{\cX_j\}_{j\in [d]}$ each of size $K$, while construct $\cH$ in a different way as follows: pick one element $x_j^\star$ from each $\cX_j$, and 
\begin{align}
    \cH = \{h: \cX \to \{\pm 1\} \mid h(x) = 1, \forall x \notin \{x_1^\star, \cdots, x_d^\star\} \}. 
\end{align}
In other words, $\cH$ shatters the set $\{x_1^\star, \cdots, x_d^\star\}$, while is always $1$ on other inputs. Clearly the VC dimension of $\cH$ is $d$. 

The learning process is divided into $d$ epochs, each of length $T/d$. During the $j$-th epoch, the adversary chooses $x_t = x_j^\star$, presents the set $\cX_j$ to the learner as the hint, and sets the following alternating sequence of $y$: 
\begin{align}
    (y_1, y_2, y_3, \cdots) = (+1, -1, +1, -1, \cdots). 
\end{align}
The best expert in $\cH$ incurs a cumulative loss of $T/2$. For the performance of \cref{alg:transductive_k_hints}, let $r_j^\star$ be the difference between the number of $+1$ and $-1$ labels in the hallucinated samples with input $x_j^\star$. One can check that if $r_j^\star \neq 0$, the learner makes half of the mistakes along the alternating sequence; if $r_j^\star = 0$, the fraction of mistakes becomes $3/4$ (\cref{alg:transductive_k_hints} cyclically predicts a wrong label and makes a random guess). Consequently, the expected regret of \cref{alg:transductive_k_hints} is lower bounded by $\Omega(T\cdot \mathbb{P}(r_j^\star = 0))$. To compute this probability, note that $r_j^\star = 2M - N$, with $N\sim \text{Bin}(T/d, n/K)$ being the number of observations $x_j^\star$ in the hallucinated data, and $M\mid N\sim \text{Bin}(N,1/2)$. Using a similar argument to Section \ref{subsec:proof_FTPL_small_n}, this probability is lower bounded by $\Omega(\min\{1, \sqrt{dK/(nT)}\})$, as desired. 

\subsection{Proof of \cref{thm:comp_lower_bound}}
The proof of \cref{thm:comp_lower_bound} uses a similar idea to \cite{HazanKoren}. There are two lower bound arguments in \cite{HazanKoren}: one reduces the problem to the Aldous' problem, and the other is based on an explicit construction of the hard instance. Although both arguments could work for our problem, we adopt the latter which corresponds to Theorem 25 of \cite{HazanKoren}. In the sequel, we will always take the domain size $|\cX| = 1/\sigma$ so that the smooth adversary becomes the usual adaptive adversary. In the next subsections, we first prove the theorem for the simpler case $d=1$, and then generalize our argument for any VC dimension $d$. 

\subsubsection{The case $d=1$.}
We first show how the argument in \cite{HazanKoren} proves the claimed $\omega(\sqrt{|\cX|})$ computational lower bound when $T = \sqrt{|\cX|} = \sqrt{1/\sigma}$. Assuming that $N\triangleq \sqrt{|\cX|}$ is an integer, we partition the domain $\cX$ into disjoint subsets $\cX_1, \cdots, \cX_N$, each of size $N$. For each $x\in \cX$, we associate two independent Rademacher variables $\varepsilon(x)$ and $\varepsilon^\star(x)$, and they are mutually independent across different $x\in \cX$. For each $i\in [N]$, the adversary chooses $x_i^\star \sim \unif(\cX_i)$, and sets the hypothesis class $\cH = \{h_x\}_{x\in \cX}$ with 
\begin{align}
    h_x(x') = \begin{cases}
    \varepsilon^\star(x') & \text{if } x = x_i^\star, x' = x_j^\star, \text{ and } i\ge j, \\
    \varepsilon(x') & \text{otherwise.}
    \end{cases} 
\end{align}
At each time $t\in [N]$, the adversary sets $x_t = x_t^\star$, and $y_t = h_{x_N^\star}(x_t) = \varepsilon^\star(x_t^\star)$. Under this setting, \cite{HazanKoren} proved the following lower bound. 
\begin{theorem}[Theorem 25 of \cite{HazanKoren}, restated]\label{thm:hazan_koren}
Given access to the ERM oracle, any proper algorithm achieving an expected regret at most $N/4$ requires $\Omega(N) = \Omega(\sqrt{|\cX|})$ running time. 
\end{theorem} 

Here by running time, we assume that each oracle call takes unit time, and maintaining each element in the input $\{(x_i, y_i)\}_{i\in I}$ to the oracle also takes unit time. We also sketch the proof idea of \cref{thm:hazan_koren} for completeness: the crucial observation is that, when the learner feeds the input $\{(x_i, y_i)\}_{i\in I}$ to the ERM oracle, the oracle can always return some $h\in \{h_0, h_{x_1^\star}, \cdots, h_{x_j^\star}\}$, where $h_0$ is any hypothesis in $\cH \backslash \{h_{x_1^\star}, \cdots, h_{x_N^\star}\}$, and $j\in [N]$ is the largest index such that $x_j^\star \in \{x_i\}_{i\in I}$. See Lemma 27 of \cite{HazanKoren} for a proof. Therefore, the label $y_t = \varepsilon^\star(x_t^\star)$ at time $t$ will look random to the learner \emph{unless} the learner has seen a function $h_{x_s^\star}$ for some $s\ge t$. By the above observation, this occurs only if the learner has set one (or more) of $\{x_s^\star\}_{s\ge t}$ as the input to the ERM oracle, but this requires one to find a random element in a size-$N$ set and thus take $\Omega(N)$ time (note that a proper algorithm only observes $\{x_1^\star, \cdots, x_{t-1}^\star\}$ at time $t$). Consequently, with $o(N)$ running time, the learner suffers from an $\Omega(N)$ loss with high probability, while the best expert incurs zero loss - giving the $\Omega(N)$ regret. 

Since the restriction of $\cH$ on any two elements $\{x, x'\}$ with $x<x'$ could only be one of the three possibilities: $\{(\varepsilon(x), \varepsilon(x')), (\varepsilon^\star(x), \varepsilon(x')), (\varepsilon^\star(x), \varepsilon^\star(x')) \}$, the VC dimension of $\cH$ is $1$. Therefore, \cref{thm:hazan_koren} gives a valid proof of \cref{thm:comp_lower_bound} when $d=1$ and $T = \sqrt{1/\sigma}$. For $T<\sqrt{1/\sigma}$, the above construction still gives the $\Omega(T)$ regret lower bound given $o(\sqrt{|\cX|})$ computational time. For general $T > \sqrt{1/\sigma}$, we make the following modification to the adversary: partition the time horizon $[T]$ into $N$ intervals $T_1,\cdots,T_N$, each of length $T/N$. For each $i\in [N]$ and $t\in T_i$, the adversary sets $x_t = x_i^\star$, and
\begin{align}
    y_t = \begin{cases}
    h_{x_N^\star}(x_t) & \text{with probability } \frac{1}{2} + \delta, \\
    - h_{x_N^\star}(x_t) & \text{with probability } \frac{1}{2} - \delta. 
    \end{cases} 
\end{align}
Consequently, the best expert $h_{x_N^\star}$ incurs an expected cumulative loss $(1/2 - \delta)T$. Meanwhile, as long as the learner cannot distinguish the distributions $\text{Bern}(1/2 + \delta)^{\otimes (T/N)}$ and $\text{Bern}(1/2 - \delta)^{\otimes (T/N)}$, she is not able to estimate $\varepsilon^\star(x_i^\star)$ based on labels $\{y_t\}_{t\in T_i}$ in the $i$-th interval. This condition is fulfilled when $\delta \asymp \sqrt{N/T}$. In addition, a similar argument for \cref{thm:hazan_koren} shows that with an $o(N)$ computational time, the learner cannot predict future $x_s^\star$ either. Therefore, any proper learner with $o(N) = o(\sqrt{|\cX|})$ computational time must incur a regret $\Omega(\delta T) = \Omega(\sqrt{T|\cX|^{1/2}})$, which is precisely the statement of \cref{thm:comp_lower_bound} for $d=1$. 

\subsubsection{General $d$.}
In this section we lift the hypothesis construction for $d=1$ to general $d$. Since $1/\sigma\ge d$, we assume that $1/(\sigma d)$ is an integer. Partition $\cX = \cup_{j=1}^d \cX_j$ each of size $|\cX|/d$, we apply the hypothesis class $\cH$ in the previous section to each $\cX_j$, and set the entire hypothesis class as
\begin{align}
    \cH_d = \left\{ h = (h_1,\cdots,h_d) \in \cH^d: h|_{\cX_j} = h_j, \forall j\in [d] \right\}. 
\end{align}
Clearly the VC dimension of $\cH_d$ is $d$. The adversary is constructed as follows: partition $[T]$ into $d$ sub-intervals $T_1,\cdots,T_d$, each of size $T/d$. For the $i$-th sub-interval, we run the subroutine in the previous section independently on $\cX_i$. Now suppose that the total runtime is $o(\sqrt{d|\cX|})$, then for at least half of the sub-intervals, the runtime during each such interval is $o(\sqrt{|\cX|/d})$. By the lower bound for $d=1$, the expected regret during each such sub-interval is 
\begin{align}
    \Omega\left( \min\left\{\frac{T}{d}, \sqrt{\frac{T}{d}\cdot \left(\frac{|\cX|}{d}\right)^{1/2}} \right\} \right) = \Omega\left( \min\left\{\frac{T}{d}, \sqrt{T\cdot \left(\frac{|\cX|}{d^3}\right)^{1/2}} \right\}  \right). 
\end{align}
Summing over at least $d/2$ such independent sub-problems, the total regret lower bound is then $\Omega(\min\{T,\sqrt{T(d|\cX|)^{1/2}}\})$, establishing the claim of \cref{thm:comp_lower_bound}.  
 \section{Statistical Upper Bounds}
\label{appendix:real-valued}

\subsection{Smoothed Online Learning}
In this section, we present a statistical upper bound achieved by a computationally inefficient algorithm. The $\sQ$ be the algorithm that  runs Hedge on a finite subset $\cH'$ on $\cH$, where $\cH'$ is a $\epsilon$-cover of $\cH$  with respect to the uniform distribution $\unif(\cX)$. The regret upper bound of this algorithm is bounded as follows.
\begin{theorem}[Statistical Upper Bound for Smoothed Online Learning] \label{thm:real-statistical}
	For any $\sigma$-smooth adversary $\sD_\sigma$, the algorithm $\sQ$ described above has regret upper bound 
	\begin{align}
		\Ex[\regret(T,\sD_\sigma,\sQ)]\in \widetilde{O}\left(
		\sqrt{Td\log\left(
			\frac{T}{d\sigma}
			\right)}+Gd\log\left(\frac{T}{d\sigma}\right)
		\right).
	\end{align}
\end{theorem}
\begin{proof}
	Let $\cH'$ be the smallest $\epsilon$-cover of $\cH$ with respect to the uniform distribution, i.e., for any $h\in\cH$, there exists a proxy $h'\in\cH'$ such that
	$\Ex_{x\sim\unif(\cX)}\left[|h(x)-h'(x)|\right]\le\epsilon$. By \cref{lemma:covering-number}, the size of $\cH$ can be upper bounded in terms of the pseudo dimension $d$:
	\begin{align}
		\log(|\cH'|)= \log \cN(\epsilon,\cH,L_1(\unif(\cX)))\le \widetilde{O}\left(d\log\!\left(\frac{1}{\epsilon}\right)\right),
	\end{align}
	where $\widetilde{O}$ hide factors that are $\log\log(1/\epsilon)$.
	Based on the net $\cH'$, we also define function class $\cG$ as follows.
	\begin{align}
		\cG=\left\{
		g_{h,{h'}}(x)=|h(x)-{h}'(x)|:
		h\in\cH,{h}'\in{\cH}'\text{ is its proxy}.
		\right\}
	\end{align}
	Now consider the following regret decomposition:
	\begin{align}
		\Ex[\regret(T)]=&\Ex\!\left[
		\sum_{t=1}^T l(\hy_t,y_t)
		-\inf_{h\in\cH}\!L(h,s_{1:T})
		\right]\\
		=&\Ex\!\left[
		\sum_{t=1}^T l(\hy_t,y_t)
		-\inf_{h'\in{\cH}'}\! L(h',s_{1:T})
		\right]+\Ex\!\left[
		\inf_{{h'}\in{\cH'}}
		L(h',s_{1:T})-\inf_{h\in\cH}L(h,s_{1:T})
		\right]
	\end{align}
	Note that the first term is precisely the regret of Hedge on the cover $\cH'$. It is thus bounded by
	\begin{align}
		\Ex\!\left[
		\sum_{t=1}^T l(\hy_t,y_t)
		-\inf_{h'\in{\cH}'}\! L(h',s_{1:T})
		\right]\le O\left(\sqrt{T\log|\cH'|}\right)\in
		\widetilde{O}\left(\sqrt{Td\log\!\left(\frac{1}{\epsilon}\right)}\right).
	\end{align}
	As for the second term, we reformulate it in terms of class $\cG$:
	\begin{align}
		&\Ex\!\left[
		\inf_{{h'}\in{\cH'}}
		L(h',s_{1:T})-\inf_{h\in\cH}L(h,s_{1:T})
		\right]=\Ex\!\left[
			\sup_{h\in\cH}\inf_{{h'}\in{\cH'}}\sum_{t=1}^T l(h'(x_t),y_t)-l(h(x_t),y_t)
		\right]\\
		\overset{(a)}{\le}&\Ex\!\left[
		\sup_{h\in\cH}\inf_{{h'}\in{\cH'}}\sum_{t=1}^T G|h(x_t)-h(x_t)|
		\right]=G\cdot\Ex_{\sD}\!\left[
		\sup_{g\in\cG}\sum_{t=1}^T g(x_t)
		\right],
		\label{tttmp1}
	\end{align}
	where $(a)$	is because the loss function $l$ has Lipschitz constant $G$.	Analogous to \cite[Claim 3.4]{haghtalab2022smoothed}, we apply the coupling argument in \Cref{lemma:coupling} to replace the adaptive sequence $x_t$s by $z_{t,k}$s that are sampled independently from the uniform distribution. Thus we obtain\begin{align}
		\Ex_{\sD}\!\left[
		\sup_{g\in\cG}\sum_{t=1}^T g(x_t)
		\right]\le T^2(1-\sigma)^K+
		\Ex_{\unif(\cX)}\left[
		\sup_{g\in\cG}
		\sum_{t=1}^T\sum_{i=1}^K g(z_{t,k})
		\right].\label{tttmp2}
	\end{align}
	The expected supremum can be further bounded in terms of the magnitude of $\cG$ (i.e., $\epsilon$) as well as the pseudo dimension of the original hypothesis class $\cH$. Using the bound in \Cref{lemma:concentration-supreme}, and together with \cref{tttmp1}, we obtain
	\begin{align}
		\Ex[\regret(T)]\le \widetilde{O}\left(
		\sqrt{Td\log\left(\frac{1}{\epsilon}\right)}+
		G\left(T^2(1-\sigma)^K+
		TK\epsilon+\sqrt{TK\epsilon d\log\!\left(\frac{1}{\epsilon}\right)}
		\right)
		\right).
	\end{align}
	In order to satisfy the condition on $n$ in \cref{lemma:concentration-supreme} and to make the failure probability of the coupling argument sufficiently small, we take $\alpha=10\log(T)$, $K=\frac{\alpha}{\sigma}$,
	$\epsilon=\Theta\left(\frac{d\sigma}{T\log(T)}\log \left(\frac{T\log(T)}{d\sigma}\right)\right)$. With this choice of parameters, we have $T^2(1-\sigma)^K=o(1)$ and
	 \begin{align}
		\Ex[\regret(T)]\le &O\left(
		\sqrt{Td\log\left(\frac{1}{\epsilon}\right)}+
		G\left(\frac{T\log(T)}{\sigma}\epsilon+\sqrt{\frac{T\log(T)}{\sigma}\epsilon d\log\!\left(\frac{1}{\epsilon}\right)}
		\right)
		\right)\\
		\le&\widetilde{O}\left(
		\sqrt{Td\log\left(
			\frac{T}{d\sigma}
			\right)}+Gd\log\left(\frac{T}{d\sigma}\right)
		\right),
	\end{align}
	as desired.
\end{proof}

\begin{lemma}[Concentration for the expected value of supreme]
	\label{lemma:concentration-supreme}
	When
	$n\ge\Omega\left(
		\frac{d}{\epsilon}\log\!\left(\frac{1}{\epsilon}\right)
		\right)$, we have\begin{align}
	\Ex_{x_{1:n}\simiid\unif(\cX)}\left[
	\sup_{g\in\cG}\sum_{i=1}^n
	g(x_i)
	\right]\le
	O\left(
	n\epsilon+\sqrt{n\epsilon d\log\!\left(\frac{1}{\epsilon}\right)}
	\right).
\end{align}
\end{lemma}
\begin{proof}
	We will use the bound on expected values of suprema of empirical processes in \cite[Theorem 3.1]{gine2006concentration}. To apply their result, the first step is to establish a bound on the $L_2(P)$-covering number of class $\cG$. Let $P_n=\frac{1}{n}\sum_{i=1}^{n}\delta_{x_i} $ be the empirical distribution based on independent samples $x_1,\cdots,x_n$.
	A similar argument to \cite[Lemma 2]{bartlett1997covering} gives us
	\begin{align}
 	\cN(\epsilon,\cG,L_2(P_n))
 	\le \cN(\frac{\epsilon}{2},\cH,L_2(P_n))^2.
	\end{align}
	Thus we obtain \begin{align}
		\log\cN(\epsilon,\cG,L_2(P_n))\le 2 \log\cN(\frac{\epsilon}{2},\cH,L_2(P_n))
		\le 2 \log\mathcal{M}(\frac{\epsilon}{2},\cH,L_2(P_n))\le O\left(
		d\log(\frac{1}{\epsilon})
		\right),
	\end{align}
	where $\mathcal{M}$ denotes the packing number and the last inequality is due to \cite[Theorem 3.1]{bartlett2008notes}. Therefore, for the function $H(x)=O(d\log x)$, we can guarantee that for any $\epsilon>1$, \begin{align}
		\log\cN(\epsilon,\cG,L_2(P_n))\le H(1/\epsilon),
	\end{align}
	satisfying the condition of \cite[Theorem 3.1]{gine2006concentration}. Therefore,
	when
	$n\ge\Omega\left(\frac{H\left(1/\epsilon\right)}{\epsilon}\right)
		=\Omega\left(
		\frac{d}{\epsilon}\log\!\left(\frac{1}{\epsilon}\right)
		\right)$,
	\cite{gine2006concentration} gives us\begin{align}
		\Ex_{\unif}\left[
		\sup_{g\in\cG}\sum_{i=1}^n\left(
		g(x_t)-\Ex[g(x_t)]
		\right)
		\right]\le
		O\left(\sqrt{n\epsilon H(1/\epsilon)}\right)
		=O\left(
		\sqrt{n\epsilon d\log\!\left(\frac{1}{\epsilon}\right)}
		\right).
	\end{align}
	Finally, since $\Ex_{\unif} g(x)\le\epsilon$ for any $g\in\cG$, we obtain\begin{align}
		\Ex\left[
		\sup_{g\in\cG}\sum_{i=1}^n
		g(x_t)
		\right]\le
		O\left(
		n\epsilon+\sqrt{n\epsilon d\log\!\left(\frac{1}{\epsilon}\right)}
		\right),
	\end{align}
	and the proof is complete.
\end{proof}

\subsection{Transductive Online Learning with $K$ Hints}
\begin{theorem}[Statistical Upper Bound for $K$-hint Transductive Learning]
    \label{thm:transductive-statistical}
    In the setting of transductive online learning with $K$ hints, there is an algorithm that achieves regret $\widetilde{O}\left(G\sqrt{Td\log(K)}\right)$, where $d$ is the pseudo dimension of the hypothesis class $\cH$, and $G$ is the Lipschitz constant of loss function $l$.
\end{theorem}
\begin{proof}
	The proof is similar to \Cref{thm:real-statistical}.
    Let $\cH'$ be an $\epsilon$-cover of $\cH$ with respect to the uniform distribution over the set of hints, i.e., $\unif(Z_{1:T})$. 
	By \Cref{lemma:covering-number}, we have $\log|\cH'|\in \widetilde{O}\left({d\log(1/\epsilon)}\right)$.
	Consider the algorithm that runs Hedge on $\cH'$. We have the following regret decomposition:
    \begin{align}
        \Ex[\regret(T)]=&\Ex\!\left[
            \sum_{t=1}^T l(\hy_t,y_t)
            -\inf_{h\in\cH}\!L(h,s_{1:T})
            \right]\\
            =&\Ex\!\left[
            \sum_{t=1}^T l(\hy_t,y_t)
            -\inf_{h'\in{\cH}'}\! L(h',s_{1:T})
            \right]+\Ex\!\left[
            \inf_{{h'}\in{\cH'}}
            L(h',s_{1:T})-\inf_{h\in\cH}L(h,s_{1:T})
            \right].
    \end{align}
    Since the first term is the regret of Hedge on $\cH'$, it is bounded by $O(\sqrt{T\log|\cH'|})\in\widetilde{O}\left(\sqrt{Td\log\!\left(\frac{1}{\epsilon}\right)}\right)$. As for the second term, use the same definition of $\cG$ and the fact that $\{x_1,\cdots,x_T\}\subset \{z_{1,1},\cdots,z_{T,K}\}$, we obtain
	\begin{align}
		\Ex\!\left[
		\inf_{{h'}\in{\cH'}}
		L(h',s_{1:T})-\inf_{h\in\cH}L(h,s_{1:T})
		\right]\le &G\cdot\Ex\!\left[
		\sup_{g\in\cG}\sum_{t=1}^T g(x_t)
		\right]\le G\cdot
            \sup_{g\in\cG}\sum_{t=1}^T\sum_{k=1}^K g(z_{t,k})
		\le	O\left(
			GTK\epsilon
			\right).
	\end{align}
	Finally, letting $\epsilon=\frac{1}{KT}$ completes the proof.
\end{proof}

\end{document}